\documentclass[12pt]{amsart}

\usepackage[utf8]{inputenc} 
\usepackage[T1]{fontenc}    
\usepackage[colorlinks=true,allcolors=blue]{hyperref} 
\usepackage{url}            
\usepackage{booktabs}       
\usepackage{amsfonts}       
\usepackage{nicefrac}       
\usepackage{microtype}      

\usepackage{algorithm, algorithmic, natbib}
\usepackage{amsmath, amsthm, amssymb, graphicx}
\usepackage[shortlabels]{enumitem}
\usepackage[margin=1.25in]{geometry}

\usepackage{thmtools}
\allowdisplaybreaks
\usepackage{bbm}
\usepackage[labelfont=it,margin=.5cm,]{caption}

\usepackage{tikz}
\usetikzlibrary{shapes}
\usepackage{graphicx}
\usepackage{caption}
\usepackage{subcaption}
\usetikzlibrary{positioning}
\usepackage{mathtools}

\newtheorem*{theorem*}{Theorem}

\usepackage{tikz}
\usetikzlibrary{shapes}
\usepackage{graphicx}
\usepackage{caption}
\usepackage{subcaption}

\theoremstyle{plain}
\newtheorem{theorem}{Theorem}
\newtheorem{lemma}[theorem]{Lemma}
\newtheorem{corollary}[theorem]{Corollary}

\newtheorem{claim}[theorem]{Claim}
\newtheorem{prop}[theorem]{Proposition}
\newtheorem{fact}[theorem]{Fact}
\newtheorem{atheorem}{Theorem}
\newtheorem*{remark}{Remark}
\newtheorem*{notation}{Notation}
\newtheorem{example}[theorem]{Example}

\theoremstyle{definition}
\newtheorem{definition}[theorem]{Definition}

\def\X{{\mathcal X}}
\def\H{{\mathcal H}}
\def\Y{{\mathcal Y}}
\def\A{{\mathcal A}}

\def\D{{\mathcal D}}
\def\Z{{\mathcal Z}}

\newcommand{\oig}{$\mathsf{OIG}$ }
\newcommand{\ad}{\mathsf{avd}}
\newcommand{\cx}{\mathcal{{X}}}
\newcommand{\ch}{\mathcal{{H}}}
\newcommand{\cy}{\mathcal{{Y}}}

\newcommand{\G}{\mathcal{G}}
\newcommand{\K}{\mathcal{K}}

\newcommand{\new}[1]{#1}

\newcommand{\dout}{\mathsf{outdeg}}
\newcommand{\Deg}{\mathsf{deg}}

\newcommand{\an}[1]{#1}

\title[A Characterization of Learnability]{A Characterization of Multiclass Learnability}
\author[Brukhim]{Nataly Brukhim}
\address{Department of Computer Science, Princeton} 
\email{nbrukhim@princeton.edu} 
\author[Carmon]{Daniel Carmon}
\address{Department of Mathematics,Technion-IIT} \email{daniel.carmon91@gmail.com}  
\author[Dinur]{Irit Dinur}
\address {Department of Computer Science,
Weizmann Institute}
\email{irit.dinur@weizmann.ac.il}
\author[Moran]{\\Shay Moran}
\address{Department of Mathematics,
Technion--IIT, and Google Research} 
\email{smoran@technion.ac.il}
\author[Yehudayoff]{Amir Yehudayoff}
\address{Department of Mathematics, Technion--IIT}\email{amir.yehudayoff@gmail.com}
\date{ }

\begin{document}
\date{}
\maketitle

\begin{abstract}
A seminal result in {learning theory}
characterizes the PAC learnability of binary classes 
through the Vapnik-Chervonenkis dimension. 
 Extending this characterization to the general multiclass setting
 has been open since the pioneering works on multiclass PAC learning in the late 1980s. 
    {This work resolves this problem: we characterize multiclass PAC learnability} through
the DS dimension, a combinatorial dimension defined by 
Daniely and Shalev-Shwartz (2014).


The classical characterization of the binary case boils down
    to empirical risk minimization. 
    In contrast, our characterization of the multiclass case involves a variety of algorithmic ideas;
    these include a natural setting we call {\it list PAC learning}.
    {In the list learning setting,}
instead of predicting a single outcome for a given 
   unseen input, the goal is to provide a short menu of predictions.


Our second main result concerns the Natarajan dimension, which 
has been a central candidate for characterizing multiclass learnability.
This dimension was introduced by Natarajan (1988) as a barrier for PAC learning. He furthered showed that it is the {only} barrier, provided that the number of labels is bounded. Whether the Natarajan dimension characterizes PAC learnability in general has been posed as an open question in several papers since. This work provides a negative answer:
we construct a non-learnable class with Natarajan dimension $1$.


For the construction, we identify a fundamental connection
between concept classes and topology (i.e., colorful simplicial complexes).
We crucially rely on a deep and involved construction of \emph{hyperbolic pseudo-manifolds}
by Januszkiewicz and {\'S}wi{\k{a}}tkowski.
It is interesting that hyperbolicity 
{is directly related to learning problems that
are difficult to solve
although no obvious barriers exist.}
{This is another demonstration of the fruitful links 
    machine learning has with different areas in mathematics.}



    \end{abstract}
    
\newpage

\section{Introduction}

Many important machine learning tasks require classification into many target classes: 
{in} {image object recognition}, the number of classes is the number of possible objects.
In {language models}, the number of classes scales with the dictionary size.
In {protein folding prediction}, the goal is to predict the 3D structures of proteins based on their 1D amino sequence.
These are real-world tasks that do not admit an a priori reasonable bound on the number of classes.
\new{Multiclass classification problems, therefore, have been attracting
interest both on the theoretical side and on the practical side;
for further reading we refer to the introduction of~\citep*{daniely2014optimal}
and references within.}

The theoretical understanding of multiclass learnability, however, is still lacking:
even in the basic Probably Approximately Correct (PAC) setting~\citep*{Valiant84}, learnability is well-understood only when the number of classes is bounded
(see e.g.~\citep*{natarajan1989learning,bendavid1995characterizations,shalev2014understanding,daniely2015multiclass}).

The fundamental theorem of PAC learning asserts the equivalence between binary classification and finiteness of the Vapnik-Chervonenkis (VC) dimension~\citep*{Vapnik68,vapnik:74,blumer:89}. 
The works of \cite*{NatarajanT88} and \cite*{natarajan1989learning}
extended Valiant's PAC framework to the multiclass setting.
They identified two natural extensions of the VC dimension:
the Natarajan dimension and the Graph dimension.
The Natarajan dimension serves as a lower bound on the sample
complexity of PAC learning,
and the Graph dimension serves as an upper bound~\citep*{NatarajanT88,natarajan1989learning}.
When the number of classes is bounded ($\lvert \Y\rvert<\infty$), 
both dimensions characterize PAC learnability.
In the unbounded case, however,
\cite*{Natarajan88up}  
showed that finite Graph dimension does not characterize PAC learnability; 
he identified a PAC learnable class 
with infinite Graph dimension
\an{(see also Example~\ref{ex:tree} below)}.
\cite*{natarajan1989learning} asked whether the Natarajan dimension characterizes learnability, and explained why standard uniform convergence techniques are not sufficient to resolve this question.

In the 1990s, \cite*{bendavid1995characterizations} and \cite*{HausslerL95} introduced a rich combinatorial framework for defining dimensions in the multiclass setting.
This framework captures as special cases many other dimensions,
including the Natarajan and Graph dimensions, the Pseudo-dimension~\citep*{Pollard90,Haussler92}, and Vapnik's dimension~\citep*{Vapnik89}.
Within this framework, \cite*{bendavid1995characterizations} exactly identified those dimensions (called {\it distinguishers}) that characterize PAC learnability when the number of classes is bounded.
This framework, however, does not capture learnability when the number of classes is unbounded, and they left this as an open problem.

More recently, a sequence of works studied general principles that guide learning in the multiclass setting \citep*{rubinstein2006shifting,DanielySS12,daniely2014optimal,daniely2015multiclass,daniely2015inapproximability}.
    These works \an{revealed} a stark contrast between 
    a bounded and an unbounded number of labels.
   One important example is that the celebrated empirical risk minimization (ERM) principle
    ceases to apply when the number of labels is unbounded~\citep*{daniely2014optimal}.

Algorithmic ideas of
\cite*{haussler1994predicting} and \cite*{rubinstein2006shifting}
lead \cite*{daniely2014optimal} to identify a {\em universal} family of transductive learning rules called  {\it one-inclusion graph} {($\mathsf{OIG}$)} algorithms.
Universality means that every learnable class can be learned by 
\oig algorithms. 
This universality and 
the combinatorial structure of \oig algorithms 
guided them to a new dimension.
We call this new dimension the {\em Daniely-Shalev-Shwartz (DS)} dimension.
They proved that finite DS dimension is a necessary condition for PAC learnability. 
But they too left the full characterization of learnability open.

\begin{remark}
We use standard terminology
from PAC learning and standard measurability assumptions (see e.g.~the textbook~\citep*{shalev2014understanding}
and references within).
All the relevant dimensions are defined and discussed in Section~\ref{sec:bg}.
\end{remark}

\subsection{Results}

{Our main result is that}
the DS dimension characterizes PAC learnability in the multiclass setting.

\begin{atheorem}[Learnability $\equiv$ Finite DS Dimension]
\label{atm:ub}
The following are equivalent for a concept class $\H\subseteq \Y^\X$:
\begin{itemize}
    \item[--] The class $\H$ is PAC learnable.
    \item[--] The DS dimension of $\H$ is finite.
\end{itemize}
\end{atheorem}

We complement this result by 
refuting the conjecture that the Natarajan dimension characterizes learnability.

\begin{atheorem}[Learnability $\not\equiv$ Finite Natarajan Dimension]
\label{atm:lb}
Finite Natarajan dimension does not characterize PAC learnability.
\end{atheorem}

The two theorems follow from more informative
results as we describe next.
Because \cite*{daniely2014optimal} proved that finite DS dimension is a necessary condition for PAC learnability,
Theorem~\ref{atm:ub} follows from the following algorithmic result.

\begin{theorem}
\label{thm:ub}
    Let $\H\subseteq\Y^\X$ be an hypothesis class with DS dimension $d<\infty$.
    \begin{description}
        \item[Realizable Case] 
There is a learning algorithm $A^{real}$ for $\H$
    with the following guarantees.
For every $\H$-realizable distribution $\D$, 
        every $\delta > 0$ and every integer $n$,
        given an input sample $S\sim \D^n$, the algorithm $A^{real}$ outputs an hypothesis $h=A^{real}(S)$ such that\footnote{The $\tilde O$ notation conceals $\mathrm{polylog}(n, d)$ factors. Logarithms in this text are always in base two.}
        $$\Pr_{(x,y) \sim \D}[h(x) \neq y]
\leq \tilde O\Bigg(\frac{d^{3/2} + \log(1/\delta)}{n}\Bigg)$$ with probability at least $1-\delta$ over $S$.
    \item[Agnostic Case] There is a learning algorithm $A^{agn}$ for $\H$
    with the following guarantees.
For every distribution $\D$, 
            every $\delta > 0$ and integer $n$,
    given an input sample $S\sim \D^n$, the algorithm $A^{agn}$ outputs an hypothesis $h=A^{agn}(S)$ such that
        $$ \Pr_{(x,y) \sim \D}[h(x) \neq y] \leq L_\D(\H) + \tilde O\Biggl(\sqrt{\frac{d^{3/2} + \log(1/\delta)}{n}}\Biggr)$$ with probability at least $1-\delta$, where $L_\D(\H) = \inf_{g \in \H}\Pr_{(x,y) \sim \D}[g(x) \neq y]$.
    \end{description}
\end{theorem}


%

Because finite DS dimension is a necessary condition for learnability,
Theorem~\ref{atm:lb} boils down to the following statement.

\begin{theorem}
\label{thm:lb}
There exists a concept class~$\H\subseteq \Y^\X$ with Natarajan dimension~$1$
and an infinite DS dimension. 
\end{theorem}

\subsection{Roadmap}
In Section~\ref{sec:bg}, 
we define the Natarajan dimension and the DS dimension.
We also introduce the reader to the DS dimension
and its basic properties.
The central goal is to explain the important links between 
the three fundamental concepts:
learnability, one-inclusion graphs, and the DS dimension.

In Section~\ref{sec:shift}, we review the shifting mechanism.
\new{Shifting is a combinatorial technique used by~\citet{haussler1995sphere} 
to analyze \oig algorithms {in the binary setting}.
\citet*{rubinstein2006shifting} later extended shifting to analyze \oig algorithms in the multiclass setting.
The multiclass setting introduces subtleties and difficulties
compared to the binary setting (see Examples~\ref{ex:1} and~\ref{ex:2} below).
To overcome these difficulties, we introduce a new combinatorial dimension, 
the {\em exponential} dimension, which might be interesting in its own right.}

Section~\ref{sec:ub} contains the proof of
the equivalence between finite DS dimension and PAC learnability.
\new{The section begins with an overview of the main challenges that arise
and the algorithmic ideas used to overcome them.
Specifically, we introduce and discuss the notion of {\em list} PAC learning,
which we believe should be of independent interest.}

In Section~\ref{sec:lb} we prove that the Natarajan dimension does not characterize PAC learnability.
This section has two parts.
One part describes a general and basic connection
between concept classes and properly colored simplicial complexes.
The second part uses a deep and beautiful construction
by~\cite*{januszkiewicz2003hyperbolic} of a simplicial complex that 
exactly meets our needs.
We provide a simplified and high-level exposition to their construction.

\section{The DS Dimension and One-inclusion Graphs} \label{sec:bg}

The prime purpose of this section is to build the bridge between 
the DS dimension and learnability.
We start with an introduction to the DS dimension,
and a description of some of its simple properties.
We continue with a description of the {\em one-inclusion graph} algorithm.
The section concludes with the story of the ``duality''
that links between the two.

\subsection{Dimensions and Pseudo-cubes}
All dimensions we consider follow a similar mechanism.
The main part is defining a notion of ``shattering''
that captures some local complexity of $\H \subseteq \Y^\X$.
{For $S \in \X^n$, 
we consider the projection $\H|_S$ of $\H$ to $S$,
and say that $\H$ shatters $S$ if $\H|_S$ is ``complex''
in some appropriate sense.
The dimension is then defined as the maximum size
($n$) of a shattered sequence
(if $\H$ shatters arbitrarily large sets then the dimension is defined to be $\infty$).}

\begin{notation}
{We consider sequences in $\X^n$ instead of subsets of $\X$,
because typically inputs to learning problems
are sequences not sets.
For $h:\X \to \Y$ and $S = (x_1,\ldots,x_n) \in \X^n$,
the projection $h|_S$ of $h$ to $S$ is thought of as 
the map from $[n]$ to $\Y$ defined by $i \mapsto h(x_i)$.
The projection of $\H$ to $S$ is
$$\H|_S = \{h|_S : h \in \H\} \subseteq \Y^n.$$
We sometimes think of $\Y^n$ as words
of length $n$ over the alphabet $\Y$.}
\end{notation}

The first and most well-known dimension is
the VC dimension.
It is defined only for binary classification problems.

\begin{definition} [VC dimension \citep*{Vapnik68}]
We say that $S \in \X^n$ is \emph{VC-shattered} by $\H \subseteq \{0,1\}^\X$ if 
$\H|_S = \{0,1\}^n$.
The VC dimension $d_{VC}(\H)$ is the maximum size of a VC-shattered sequence.
\end{definition}

When $|\Y| > 2$, there are many ways to extend the VC dimension.
One of the first extensions of the VC dimension to the multiclass
setting is the Natarajan dimension.
The relevant shattering is ``containing a copy of the Boolean cube''.

\begin{definition} [Natarajan dimension \citep*{natarajan1989learning}]
{We say that $S \in \X^n$ is \emph{N-shattered} by $\H \subseteq \Y^\X$ if there exist $f,g: [n] \rightarrow \Y$ such that for every $i \in [n]$ we have
$f(i) \neq g(i)$, and }
$$\H|_S \supseteq \{f(1),g(1)\}
\times \{f(2),g(2)\} \times \ldots
\times \{f(n),g(n)\}  .$$
The Natarajan dimension $d_{N}(\H)$ is the maximum size of an N-shattered sequence.

%
\end{definition}

%
%
%

%

What is the ``correct analog'' of the Boolean cube
for larger alphabet sizes?
There are many possible answers.
The starting point of the definition of the DS dimension is viewing the Boolean cube as a graph.
The vertex-set of the graph is $\{0,1\}^d$.
The edges of the graph are defined as follows.
For every vertex $v \in \{0,1\}^d$
and for every direction $i \in [d]$,
there is a (single) neighbor $u$ of $v$ in direction $i$
(that is, $u(i) \neq v(i)$ and $u(j)=v(j)$ for all $j \neq i$).
    This perspective can be naturally applied to non-binary concept classes.

\begin{definition}[Pseudo-cube]
A class $\H\subseteq \Y^d$ is called a {\em pseudo-cube} 
of dimension~$d$ if it is non-empty, finite and for every $h\in \H$ and $i \in [d]$,
there is an $i$-neighbor~$g \in \H$ of $h$
(i.e., $g(i)\neq h(i)$ and $g(j)=h(j)$ for all $j \neq i$). 
\end{definition}

When $\Y = \{0,1\}$, 
the two notions ``Boolean cube'' and ``pseudo-cube'' coincide:
The Boolean cube $\{0,1\}^d$ is of course a pseudo-cube. 
Conversely, every pseudo-cube $\H\subseteq \{0,1\}^d$ is the entire 
Boolean cube $\H=\{0,1\}^d$. 
When $\lvert \Y\rvert >2$,
the two notions do not longer coincide.
Every copy of the Boolean cube is a pseudo-cube,
but there are pseudo-cubes that are not Boolean cubes; see Figure~\ref{fig:2dim} for an example.
\an{The example in the figure uses a dual perspective.
The functions (words) in the class are the edges of the graph,
and the alphabet symbols are the vertices of the graph.
This dual perspective is important and useful.
We discuss it in more detail in Section~\ref{sec:lb}.}

\begin{figure}[h!]
\centering
\begin{subfigure}{.3\textwidth}
  \centering
  
  \medskip
  
  \medskip
  
  \medskip
  
\begin{tikzpicture}
 \draw (0,0) -- (1,1) -- (2,0) -- (1,-1) -- (0,0);
 
 \tikzset{bluecirc/.style={circle, draw=black, fill=teal!50, inner sep=0pt,minimum size=13pt}}
\tikzset{violsqur/.style={draw=black, fill=blue!50, inner sep=0pt,minimum size=13pt}}

\node[bluecirc] at (0,0) {$a$};
\node[bluecirc] at (2,0) {$c$};

\node[violsqur] at (1,1) {$b$};
\node[violsqur] at (1,-1) {$d$};

%
%
%
%

\end{tikzpicture}
\begin{align*}
\ \\
ab \\
cb \\
cd \\
ad \\
\ \\
\end{align*}

\end{subfigure}%
\begin{subfigure}{.3\textwidth}
  \centering
\begin{tikzpicture}

\tikzset{bluecirc/.style={circle, draw=black, fill=teal!50, inner sep=0pt,minimum size=13pt}}
\tikzset{violsqur/.style={draw=black, fill=blue!50, inner sep=0pt,minimum size=13pt}}

 \draw (0,0) -- (1,0.6) -- (2,0) ;
 \draw (0,0) -- (0,-1) -- (1,-1.6) -- (2,-1) -- (2,0) ;


\node[bluecirc] at (0,0) {$1$};
\node[bluecirc] at (2,0) {$3$};
\node[bluecirc] at (1,-1.6) {$5$};

\node[violsqur] at (1,0.6) {$2$};
\node[violsqur] at (0,-1) {$6$};
\node[violsqur] at (2,-1) {$4$};

%
%
%
%

\end{tikzpicture}
\begin{align*}
12 \\
32 \\
34 \\
54 \\
56 \\
16
\end{align*}

\end{subfigure}%
 \caption{A $2$-dimensional pseudo-cube (on the right) 
 that is not isomorphic to the $2$-dimensional Boolean cube (on the left). 
The labels $\Y$ are the vertices
($4$ label on the left, and $6$ labels on the right).
The words in $\H \subset \Y^2$ are the edges ($4$ words on the left,
and $6$ words on the right).
For each word, the circle vertex appears as the {first symbol,
and the square appears as the second symbol}. 
 }\label{fig:2dim} 
\end{figure}
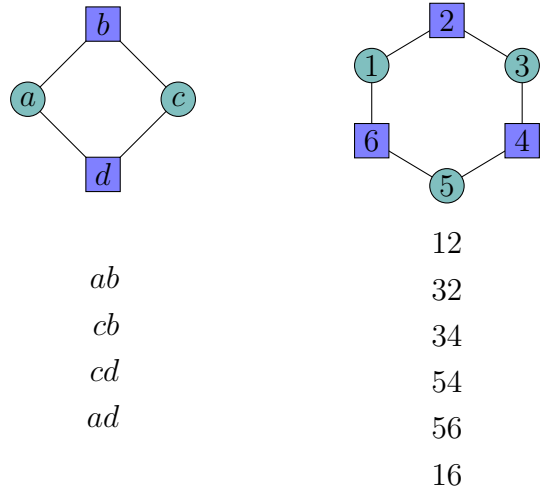

The DS dimension is defined by containing pseudo-cubes
(the original definition uses a slightly different language,
but it is equivalent).

\begin{definition}[DS dimension \citep*{daniely2014optimal}]
We say that $S \in \X^n$ is
{\em $DS$-shattered} by $\H\subseteq \Y^\X$ 
if $\H|_S$ contains an $n$-dimensional
pseudo-cube.
The DS dimension $d_{DS}(\H)$ is the maximum size of a DS-shattered sequence.
\end{definition}


How different are pseudo-cubes than Boolean cubes? 
Or, more formally, are there $d$-dimensional pseudo-cubes with Natarajan dimension $<d$?
The hexagon in Figure~\ref{fig:2dim} is a $2$-dimensional pseudo-cube whose Natarajan dimension is $1$. There are, in fact, many other such constructions, even in the $2$-dimensional case.

The following \an{example} provides a complete description of $2$-dimensional pseudo-cubes 
with Natarajan dimension $1$ using the language of graph theory.
We omit the proof because in Section~\ref{sec:lb} we derive generalizations 
to arbitrary dimensions.

\begin{example}\label{obs:2dim}
For every bipartite graph $G=(L , R,E)$ the set 
\[B(G):=\bigl\{(u,v) \in L \times R : \{u,v\}\in E\bigr\}\]
is a 2-dimensional pseudo-cube if and only if $G$ contains no leaves.
Conversely, for every $B\subseteq \Y^2$,
the bipartite graph 
\[G(B):=\bigl(L=\Y\times\{0\},R=\Y\times\{1\}, E=\bigl\{\{(y_0,0),(y_1,1)\} : (y_0,y_1)\in B\bigr\}\bigr)\]
contains no leaves if and only if $B$ is a pseudo-cube.
\an{The claim is that}
a $2$-dimensional pseudo-cube $B\subseteq \Y^2$ has Natarajan dimension~$2$ iff the corresponding bipartite graph $G(B)$ has a $4$-cycle.
\end{example}

The above demonstrates that $2$-dimensional pseudo-cubes are rather simple combinatorial objects.
The landscape in higher dimensions is significantly richer. Figure~\ref{fig:3dimpc} depicts a $3$-dimensional pseudo-cube with Natarajan dimension~$1$. This pseudo-cube arises from a triangulation of the 
plane; a hint towards the topology that is used in Section~\ref{sec:lb} to prove Theorem~\ref{thm:lb}.

\begin{figure}[h!]
\centering

\begin{tikzpicture} 
\tikzset{bluecirc/.style={circle, draw=black, fill=blue!50, inner sep=0pt,minimum size=8pt}}
\tikzset{violsqur/.style={draw=black, fill=violet!50, inner sep=0pt,minimum size=8pt}}
\tikzset{tealtri/.style={isosceles triangle,
    isosceles triangle apex angle=60,rotate=90, draw=black, fill=teal!50, inner sep=0pt,minimum size=8pt}}

\draw[thick] (0.5,{3 *sqrt(3)/2}) -- (2,{6 *sqrt(3)/2}) -- (5,{6 *sqrt(3)/2}) --
(6.5,{3 *sqrt(3)/2}) -- (5, 0) -- (2, 0) -- (0.5,{3 *sqrt(3)/2});


\node [rotate=60] at (1.25,{4.5 *sqrt(3)/2}) {$\gg$};
\node [rotate=60] at ({5.75},{1.5 *sqrt(3)/2}) {$\gg$};

\node [rotate=60] at (1.25,{4.5 *sqrt(3)/2}) {$\gg$};
\node [rotate=60] at ({5.75},{1.5 *sqrt(3)/2}) {$\gg$};

\node [rotate=-60] at ({5.75},{4.5 *sqrt(3)/2}) {$\supset$};
\node [rotate=-60] at ({1.25},{1.5 *sqrt(3)/2}) {$\supset$};

\node at (3.5,{6 *sqrt(3)/2}) {$>$};
\node at (3.5,0) {$>$};

\path[clip] (0.5,{3 *sqrt(3)/2}) -- (2,{6 *sqrt(3)/2}) -- (5,{6 *sqrt(3)/2}) --
(6.5,{3 *sqrt(3)/2}) -- (5, 0) -- (2, 0) -- (0.5,{3 *sqrt(3)/2});
 
\foreach \x in {0.5,...,6.5}{
    \foreach \y in {0,2,...,6}{
      \pgfmathsetmacro\xx{\x - 0.5*Mod(\y+1,2)};
      \pgfmathsetmacro\yy{\y*sqrt(3)/2};
      \pgfmathsetmacro\yyy{(\y+1)*sqrt(3)/2};
      
     \draw (\xx,\yy) -- ++ (60:1) -- ++ (-60:1) -- cycle;
    \draw (\x,\yyy) -- ++ (60:1) -- ++ (-60:1) -- cycle;
    
                \pgfmathparse{int(mod(\x,3))} 
             \ifnum0=\pgfmathresult\relax
                     \node[violsqur]  at (\xx,\yy) {};                           
                     \node[bluecirc]  at (\x,\yyy) {};
                \else  
                 \ifnum1=\pgfmathresult\relax
                     \node[tealtri]   at (\xx,\yy) {};
                     \node[violsqur]  at (\x,\yyy) {};
                \else  
                    \node[bluecirc]  at (\xx,\yy) {};
                    \node[tealtri]  at (\x,\yyy) {};
                \fi
               \fi
    }
}

\end{tikzpicture}
\caption{An example of a $3$-dimensional pseudo-cube with Natarajan dimension $1$. The words in the pseudo-cube are the triangles
(there are 54 of them).
The labels are the vertices
(there are $\tfrac{3 \cdot 54}{6} = 27$ of them).
The vertices are colored by $3$ colors: circle, triangle and square.
In each word, the circle vertex appears as the first symbol,
the triangle vertex as the second,
and the square vertex as the third.
Opposite sides of the hexagon are identified,
as the picture indicates.
The pseudo-cube property holds because each triangle
has three neighboring triangles
that are obtained by switching one vertex from each color.
The Natarajan dimension is $1$ because
there is no square (a cycle of length four) in the graph
so that its vertices have alternating colors.
For more details,
see Section~\ref{sec:lb}.}\label{fig:3dimpc}
\end{figure}
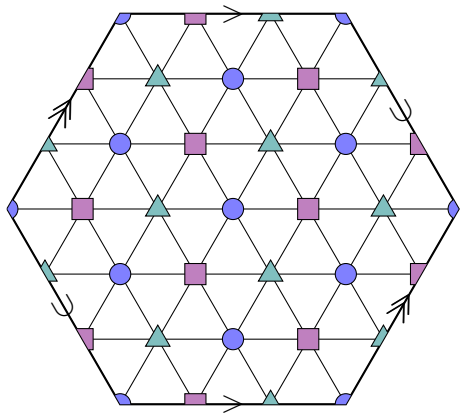

The condition that pseudo-cubes are {\em finite} is surprisingly important.
Without it, the DS dimension does not characterize learnability,
as the following example shows.

\begin{example}
\label{ex:tree}
There is an infinite learnable class $\H^{tree}$ over $\X = \mathbb{N}$
so that for each $x \in \X$ and $h \in \H^{tree}$,
there is $g \in \H^{tree}$ that agrees with $h$ on all points besides~$x$.
{But the DS dimension of this class is $1$,
so it is learnable \an{(by our main result)}.}
This class can be thought of as a directed tree
whose edges are directed towards the root;
\an{Figure~\ref{fig:2dim} illustrates a similar class for the case
$\X=[3]$.}
The root is the all-zeros function.
Each $h$ in the tree has $\X$ in-going edges;
for each $x \in \X$,
there is an edge towards $h$ from the function that is equal to $h$
on $\X \setminus \{x\}$,
and is equal to a new and unique alphabet symbol at~$x$.
Every alphabet symbol $y \in \Y$
has a {\em depth}; it is the minimum distance from the root
of a word that $y$ appears in.
{The DS dimension is less than two for the following reason.}
For every $S \in \X^2$ and every finite $\H_0 \subset \H^{tree}|_S$, we can choose $y \in \Y$ with maximum depth 
among all symbols that appear in $\H_0$.
Let $h_0$ be an element in $\H_0$ that contains $y$.
The vertex $h_0$ does not have two neighbors,
so $\H_0$ is not a pseudo-cube.
\end{example}

\begin{figure}
\centering
\begin{tikzpicture}
\node (empty) [scale=.9] {$(0,0,0)$};

\node (1) [below left = .7cm and 1.2cm of empty,scale=.8] {$(1,0,0)$};
\node (2) [right=1.2cm of 1,scale=.8] {$(0,2,0)$};
\node (3) [right=1.2cm of 2,scale=.8] {$(0,0,3)$};

\draw (empty)--(1);
\draw (empty)--(2);
\draw (empty)--(3);

\node (00) [below left = .7cm and 1.2cm of 1,scale=.8] {$(4,0,0)$};
\node (01) [right=.3cm of 00,scale=.6] {$(1,5,0)$};
\node (02) [right=.3 cm of 01,scale=.6] {$(1,0,6)$};

\draw (1)--(00);
\draw (1)--(01);
\draw (1)--(02);

\node (10) [below left = .7cm and .1cm of 2,scale=.6] {$(7,2,0)$};
\node (11) [right=.1cm of 10,scale=.6] {$(0,8,0)$};
\node (12) [right=.1cm of 11,scale=.6] {$(0,2,9)$};

\draw (2)--(10);
\draw (2)--(11);
\draw (2)--(12);

\node (31) [below  = .7cm  of 3,scale=.6] {$(10,0,3)$};
\node (32) [right=.1cm of 31,scale=.6] {$(0,11,3)$};
\node (33) [right=.1cm of 32,scale=.6] {$(0,0,12)$};

\draw (3)--(31);
\draw (3)--(32);
\draw (3)--(33);

\node (000) [below  left = .4cm and .5cm of 00,scale=.6] {$(13,0,0)$};
\node (001) [right=.4cm of 000,scale=.6] {$(4,14,0)$};
\node (002) [right=.4cm of 001,scale=.6] {$(4,0,15)$};

\draw (00)--(000);
\draw (00)--(001);
\draw (00)--(002);



\node (vdots3) [below =.3cm of 32] {$\vdots$};
\node (vdots11) [below =.3cm of 11] {$\vdots$};
\node (vdots02) [below =.3cm of 02] {$\vdots$};
\node (vdots001) [below =.1cm  of 001] {$\vdots$};
\end{tikzpicture} \caption{An example of an infinite class $\H^{tree} \subseteq \Y^\X$ with  $\X=[3]$  and $ \Y =  \mathbb{N}$.}\label{fig:2dim}
\end{figure}
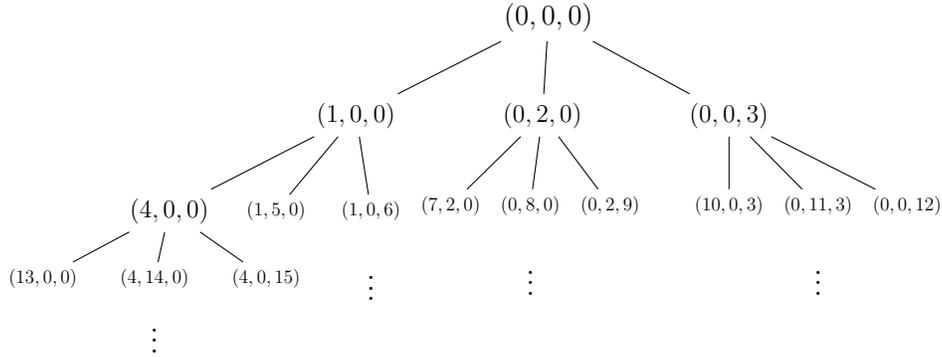

%

\subsection{The One-Inclusion Graph} \label{sec:orientations}
This subsection introduces an important combinatorial 
abstraction of learning algorithms. 
The idea is to translate a learning problem
to the language of graph theory.

\begin{definition}[One-inclusion Graph~\citep*{haussler1994predicting,rubinstein2006shifting}]
The {\it one-inclusion} graph  
of $\H\subseteq \Y^n$ is a hypergraph
$\mathcal{G}(\H)=(V,E)$ that is defined as follows.\footnote{We use
the term ``one-inclusion graph'' although it is actually a hypergraph.}
The vertex-set is $V=\H$.
For each  $i\in [n]$ and $f:{[n]\setminus\{i\}} \to \Y$,
let $e_{i,f}$ be the set of all $h\in \H$ that agree with $f$ on $[n]\setminus \{i\}$.
{The edge-set is
\begin{equation} \label{eq:edge_set}
 E = \big\{(e_{i,f},i) : i\in [n], f:{[n]\setminus\{i\}} \to \Y,
 e_{i,f} \neq \emptyset  \big\} .
\end{equation}
We say that the edge $(e_{i,f},i) \in E$ contains the vertex $v$,
and write $v \in (e_{i,f},i)$,
if $v \in e_{i,f}$.
The size of the edge $(e_{i,f},i)$ is defined to be
$|(e_{i,f},i)| := |e_{i,f}|$.}
\end{definition}

\begin{remark}
{The edge-set consists of pairs $(e,i)$,
where $e$ is a set of vertices and~$i$ is the {\em direction} of the edge.
\an{It is convenient that the ``name'' of an edge
also tells us its direction.}
Edges could be of size one,
and each vertex $v$ is contained in exactly $n$ edges.
This is not the standard structure of edges in hypergraphs,
but we use this notation because it provides a better model
for learning problems.}
\end{remark}

The one-inclusion graph 
leads to a simple but useful toy model for 
{transduction in} machine learning.

\begin{example}[Toy Model]\label{ex:toyM}
The learning game is played over a one-inclusion graph $(V,E)$.
The input to the problem is an edge.
The input edge $e$ is generated by first choosing a vertex $v_*$ 
from some unknown distribution over $V$,
and then choosing $e$ to be a uniformly random edge containing $v_*$.
The goal is to output a vertex $u$ that is equal to $v_*$
with as high probability as possible.
\end{example}

Learning algorithms in this toy model are orientations.
\begin{definition}
An {\em orientation} of the hypergraph $(V,E)$ is a mapping $\sigma :E\to V$
such that $\sigma(e)\in e$ for each edge $e\in E$. 
\end{definition}

Every (deterministic) learning algorithm
defines an orientation, and vice versa.
The success probability of the algorithm is determined
by the out-degrees of the orientation.
The {\em out-degree} of $v \in V$ {in} $\sigma$ is 
\begin{equation} \label{eq:out_deg}
\dout(v;\sigma) = \lvert\{e: v \in e \text{ and } \sigma(e)\neq v\}\rvert.    
\end{equation} 
The maximum out-degree of $\sigma$ is 
\begin{equation} \label{eq:max_out_deg}
    \dout (\sigma) = \sup_{v\in V} \ \dout(v;\sigma).
\end{equation} 
{There is a certain 
duality between orientations and the DS dimension,
as the following two lemmas demonstrate.
This duality is the basic link between the DS dimension
and learnability.}

\begin{lemma}
\label{lem:PCisLB}
If $\H\subseteq \Y^d$ has DS dimension $d$,
then $\dout(\sigma)\geq \tfrac{d}{2}$ for every orientation $\sigma$ of $\G(\H)$.
\end{lemma}

\begin{lemma} \label{lem:orientation}
If $\H\subseteq \Y^{d+1}$ has DS dimension $d$,
then there exists an orientation $\sigma$ of $\G(\H$) with 
$\dout(\sigma) \leq d$.
\end{lemma}

\begin{proof}[Proof of Lemma~\ref{lem:PCisLB}]
We prove the stronger assertion that
if $\H\subseteq \Y^d$ is a pseudo-cube then every orientation $\sigma$ of $\H$ satisfies that $\dout(\sigma)\geq \tfrac{d}{2}$.
In a pseudo-cube, each $h$ has a neighbor in each of the $d$ directions,
and every edge $e\in E$ has size $\lvert e\rvert \geq 2$
so that $|e|-1 \geq \tfrac{|e|}{2}$.
Even the average out-degree is at least~$\tfrac{d}{2}$:
for every orientation $\sigma$,
\begin{align*}
\frac{1}{\lvert V\rvert}\sum_{v\in V}\dout(v;\sigma) &=
\frac{1}{\lvert V\rvert}\sum_{e\in E}\lvert e\rvert -1 \\ 
&\geq 
\frac{1}{\lvert V\rvert}\sum_{e\in E}\frac{\lvert e\rvert}{2} \\
&= \frac{1}{2\lvert V\rvert}\cdot d\lvert V\rvert = \frac{d}{2}. 
\end{align*}
{This finishes the proof because the maximum is at least the average.}
\end{proof}

\begin{proof}[Proof of Lemma~\ref{lem:orientation}]
We start by analyzing the case that $\H$ is finite
(similarly to~\citep*{daniely2014optimal}).
The orientation is constructed inductively and greedily as follows.
The base of the induction is the case $|\H|=1$.
In this case, {all edges are oriented towards the single vertex,}
so the claim trivially holds.
For the inductive step, assume $|\H|>1$.
Let $\G=(V,E)$ be the {one-inclusion graph} associated with $\H$.
Because the DS dimension of $\H$ is $d$, 
there must be $h\in \H$ 
so that {the size of $\{e \in E: h \in e, |e|>1\}$
is at most $d$.}
Let $\H'$ be $\H$ after deleting~$h$.
By definition, the DS dimension of $\H'$ is at most $d$.
Let $\G' = (V',E')$ be the hyper-graph associated with $\H'$.
{Edges in $E'$ are obtained from edges in $E$
by deleting $h$.
There is at least one singleton edge of size $1$ that contains $h$
in $E$.
This edge does not appear in $E'$.}
By \an{the induction hypothesis}, there is an orientation $\sigma' : E' \to V'$
with maximum out-degree at most $d$.
{Every edge in $E'$ corresponds to an edge in $E$.
The only edges in $E$ that do not have
counterparts in $E'$ are the singleton edges that contain $h$.
Let $\sigma: E \to V$ extend $\sigma'$ as follows. 
Every edge in $E$ that has a counterpart in $E'$
is directed in $\sigma$ as in $\sigma'$.
All other (singleton) edges are directed towards $h$.
The out-degree of vertices in $\G'$ does not change,}
and the out-degree of $h$ is at most $d$.
So, the out-degree of $\sigma$ is at most $d$ as required.

The case when $\H$ is infinite is handled using a compactness argument. 
Because we could not find a proper reference,
we provide the short (but not entirely trivial) proof
in Appendix~\ref{app:comapct}.

\end{proof}

%
%


%
    
\subsection{The One-Inclusion Graph Algorithm} \label{subsec:oi}
The one-inclusion graph captures a model for {transduction in machine} learning
(Example~\ref{ex:toyM}).
A key observation of~\citet*{haussler1994predicting}
is that this model captures an essential ingredient of general PAC learnability;
see also~\citep*{rubinstein2006shifting,daniely2014optimal,alon2021theory}.
In a nutshell, good orientations of the one-inclusion graph 
yield good learning algorithms.

\begin{algorithm}
\caption{The one-inclusion algorithm $\A_\H$ for $\ch\subseteq\cy^{\cx}$} \label{algo:one_inc}
\begin{flushleft}
  {\bf Input:} An $\H$-realizable sample $S = \big((x_1, y_1),\ldots,(x_n, y_n)\big)$. \\
{\bf Output:} A hypothesis $\A_{\H}(S)=h_S:\X \to \Y$. \\
\ \\
For each $x \in \X$, the value $h_S(x)$ is computed as follows.
\end{flushleft}
\begin{algorithmic}[1]
\STATE Consider the class of all patterns over the \emph{unlabeled data}
{$\H|_{(x_1,\ldots,x_n,x)} \subseteq \Y^{n+1}$}.
\STATE Find an orientation $\sigma$ of $\mathcal{G}(\H|_{(x_1,\ldots,x_n,x)})$ that minimizes the maximum out-degree.
\STATE Consider the edge in direction $n+1$ defined by $S$;
let
\[e =\{ h \in \H|_{(x_1,\ldots,x_n,x)} :  \forall i \in [n]  \ \ 
h(i) =y_i\}.\]
\STATE Let $h' = \sigma((e,n+1))$.
\STATE Set $h_S(x) = h'(n+1)$.
\end{algorithmic}
\end{algorithm}

The one-inclusion graph ($\mathsf{OIG}$) algorithm $\A_{\H}$ is presented in Algorithm~\ref{algo:one_inc}.
The algorithm gets as input a realizable training sample
$S = ((x_1,y_1),\ldots,(x_n,y_n))$ as well as
an additional test point $x$.
Its goal is to provide a good prediction for the label of $x$.
The main idea is to translate this problem to the toy model.
Use the unlabelled data $x_1,\ldots,x_n$ and $x$
to build the one-inclusion graph of $\H|_{(x_1,\ldots,x_n,x)}$.
The labels $y_1,\ldots,y_n$ now define
an edge in the graph.
An orientation of the graph provides the prediction 
for the label of $x$.

The crucial point is that an orientation 
    with small maximum out-degree yields small error.
    This follows by a simple and clever \emph{leave-one-out} argument (see e.g.\ \citep*{haussler1994predicting}).
    The argument exploits the underlying symmetry as 
    we now explain.

    Let $S \sim \D^n$ be the input sample 
    and let $(x,y) \sim \D$ be the test point (chosen independently of $S$).
    We can generate the joint distribution of $(S,(x,y))$ in a different way.
    We can choose $S' \sim \D^{n+1}$ and independently choose $I$ from the uniform 
    distribution $U(n+1)$ on $[n+1]$. 
    Let 
  $$S'_{-I}=((x'_1,y'_1),\ldots,(x'_{I-1},y'_{I-1}),(x'_{I+1},y'_{I+1}),\ldots ,
  (x'_{n+1},y'_{n+1}))$$ be the sample $S'$ after deleting its $I$ entry.
    The distribution of $(S'_{-I},(x'_I,y'_I))$ is identical to that of $(S,(x,y))$.

\begin{fact}[Leave-one-out] \label{fact:sym}
Let $\D$ be a distribution over a set $\mathcal{Z}$ and let $n>0$ be an integer.
For every event $E\subseteq \mathcal{Z}^{n+1}$,
\[
\Pr_{(S,Z) \sim \D^{n+1}}\Bigl[(S, Z) \in E\Bigr] = 
\Pr_{(S',I) \sim \D^{n+1} \times U(n+1)}  \Bigl[ (S'_{-I}, S'_I)\in E \Bigr] .
\]
\end{fact} 

The one-inclusion graph together with
the leave-one-out argument lead to a formal connection
between good orientations and PAC prediction error.

\begin{prop}
Let $\D$ be an $\H$-realizable distribution and let $n>0$
be an integer.
Let $M$ be an upper bound
on the maximum out-degree 
of all orientations chosen by $\A_\H$.
The prediction error can be bounded as
$$\Pr_{(S,(x,y)) \sim \D^{n+1}} \bigl[h_S(x) \neq y \bigr] \leq \frac{M}{n+1},$$
where $h_S = \A_\H(S)$.
\end{prop}

\begin{proof}
By Fact~\ref{fact:sym},
\[
\Pr\bigl[h_S(x) \neq y \bigr] = 
\Pr  \bigl[ h_{S'_{-I}}(x'_I) \neq y'_I \bigr] .
\]
The prediction error is small,
as long as the maximum out-degree is small:
for every fixed $S' = ((x'_1,y'_1),\ldots,(x'_{n+1},y'_{n+1}))$,
\begin{align*}
\Pr_{I }\bigl[h_{S'_{-I}}(x'_I) \neq y'_I \bigr] 
    &=
    \frac{1}{n+1}\sum_{i=1}^{n+1} 1\bigl[ h_{S'_{-i}}(x'_i) \neq y'_i \bigr]\\
   &=
   \frac{1}{n+1}\sum_{i=1}^{n+1} 1\bigl[ \sigma(e_i) \neq y'_i \bigr] \\
   &=   \frac{\dout(y';\sigma)}{n+1} ,
\end{align*}
    where $y' = (y'_1,\ldots,y'_{n+1})$
    is a vertex the one-inclusion graph,
    and $e_i$ is the edge {in the $i$'th direction
    containing $y'$.}
    \end{proof}

The final piece we present in this section is
that a bound on the DS dimension leads to
non-trivial prediction guarantees for PAC learning.
This \an{rather} weak prediction capability 
is the starting point of our general PAC learning algorithm.
It will be significantly enhanced in Section~\ref{sec:ub} below.

\begin{claim} \label{claim:good_pt}
Let $\H \subseteq \Y^\X$ be so that $d=d_{DS}(\H)  < \infty$.
\an{Let $\A_\H$ be Algorithm~\ref{algo:one_inc}.}
For every $\H$-realizable sample $S'=((x'_1,y'_1),\ldots,(x'_{d+1},y'_{d+1}))$,
there exists $i \in [d+1]$ such that $h_{S'_{-i}}(x'_i) =y'_i$,
where $h_{S'_{-i}} = \A_{\H}(S'_{-i})$.
\end{claim}

\begin{proof}
Let $\H'  = \H|_{(x'_1,\ldots, x'_{d+1})}$. 
Think of $y'=(y'_1,\ldots,y'_{d+1})$ as a vertex in $\G(\H')$.
Let $\sigma$ denote the orientation 
that minimizes the maximum out-degree of $\G(\H')$ chosen by $\A_\H$.
Lemma~\ref{lem:orientation} implies that the maximum out-degree
of $\sigma$ is at most~$d$.
Let $e_i$ be the edge {in the $i$'th direction containing $y'$.}
\an{For every $i \in [d+1]$,
we have $h_{S'_{-i}}(x'_i) \neq y'_i \ \Leftrightarrow \
\sigma(e_{i}) \neq y'$.
So,
\[
\sum_{i=1}^{d+1} 1\bigl[ h_{S'_{-i}}(x'_i) \neq y'_i \bigr]=\sum_{i=1}^{d+1} 1\bigl[ \sigma(e_{i}) \neq y' \bigr] =   \dout(y';\sigma)\leq d .
\]
It follows that} there must exist $i$ such that $h_{S'_{-i}}(x'_i) = y'_i$.
\end{proof}

\section{Shifting and Orientations}
\label{sec:shift}

\new{In this section we use a powerful combinatorial technique called shifting to derive good orientations.
This links the general discussion of one-inclusion graphs from the previous section,
with the learning algorithm we use to prove Theorem~\ref{atm:ub} in the next section.}
The main result of this section is that
the out-degree of optimal orientations can be
controlled by the Natarajan dimension and the 
number of labels.\footnote{Here and below
we did not attempt to optimize the constants.}

\begin{lemma}\label{lem:wl:rub}
Let $\H \subseteq [p]^n$ be a class with Natarajan dimension $d_N < \infty$.
Then, there exists an orientation $\sigma$ of $\G(\H)$ with maximum out-degree \[\dout(\sigma)\leq 20 d_N \log p .\]
\end{lemma}   

\newcommand{\ppd}{\Psi_Pdim}
\renewcommand{\S}{\mathbb{S}}

The key technique we use is shifting~\citep*{haussler1995sphere,rubinstein2006shifting}.
Shifting is a way to simplify the structure of a hypothesis class, while 
controlling important properties.
\an{Intuitively, it is the operation of ``pushing a concept class downward''.
Think of $[p]$ as totally ordered by the standard order on $\mathbb{N}$.
The set $[p]^n$ becomes a poset with the partial order
$h \leq g$ iff $h(i) \leq g(i)$ for all $i$.}

\begin{definition}[Shifting]
\label{def:shift}
{Let $\H\subseteq [p]^n$ and let $i \in [n]$.
The shifting operator in the $i$'th direction $\S_i$
maps $\H$ to its shifted version $\S_i(\H)$ as follows.
Shifting is first defined on edges.
For $f : [n] \setminus \{i\} \to [p]$,
let $e_{f}$ be the collections of $h \in \H$
that agree with $f$ on $[n] \setminus \{i\}$.
The shifting $\S_i(e_f)$ is obtained by ``pushing
$e_f$ downward''; namely,
$\S_i(e_f)$ is the collection of all
$g \in [p]^n$ that agree with $f$ on $[n] \setminus \{i\}$
and $1 \leq g(i) \leq |e_f|$.
The shifting of $\H$ is the union of all shifted edges}
$$\S_i(\H) = \bigcup_{f} \S_i(e_f) \subseteq [p]^n.$$
\end{definition}

\an{Let us provide a different view point on this important operation.
Fix $j \neq i$, and 
partition all edges in the $j$'th direction according
to their projection to $[n] \setminus \{i,j\}$.
Fix $f \in \H|_{[n] \setminus \{i,j\}}$,
and consider all vertices that agree with $f$ on $[n] \setminus \{i,j\}$.
Encode this data by the $p\times p$ Boolean matrix $M_{f}$
defined by $M_f(a,b) = 1$
iff adding $a,b$ to $f$ in positions $i,j$ leads to a word in $\H$.
The $1$-entries in the matrix correspond to words in $\H$ that agree with $f$.
Every row in the matrix corresponds to the (possibly empty or singleton) set of words that differ in the $j$'th coordinate. 
Rows with at least {one $1$-entry} correspond to edges in the 
one-inclusion graph.
The matrix offers a nice viewpoint on shifting.
Shifting is performed by pushing all the $1$-entries ``upwards''.
Here is an example of shifting six words
over an alphabet of size four:}
\begin{align*}
\left[ \begin{array}{cccc}
 1 & \ &1& \\
 & & & \\
 &&1& \\
1  &&1& 1\\
\end{array} \right]
\Longrightarrow 
\left[ \begin{array}{cccc}
 1 & \  &1& 1\\
1 & &1 & \\
 &&1& \\
  &&& \\
\end{array} \right]
\end{align*}

Repeatedly applying the shifting operator in various
directions leads to a fixed point $\H_*$ of these operators;
\an{that is, $\S_i(\H_*) = \H_*$ for all $i$.}
This must happen in a finite number of steps,
because when a change is made
the total sum of all entries strictly decreases.
The fixed points of shifting are classes 
that are closed downwards
\an{(that is, if $h$ is in a fixed-point $\H_* \subseteq [p]^n$ 
and $g \leq h$ then $g \in \H_*$).}

In the binary setting, \citet{haussler1995sphere}
proved that shifting does not increase the VC dimension,
and that it does not decrease the number of edges in the one-inclusion graph.
This allows to elegantly bound from above the edge density by the VC dimension.

\an{In the multiclass setting, 
\citet*{rubinstein2006shifting} used
the Pollard dimension~\citep{Pollard90} to control the behavior of multiclass shifting;
the Pollard dimension
provides a natural mechanism for moving from the multiclass setting
to the binary setting.
But the Pollard dimension and other standard dimensions can grow during shifting;
see Example~\ref{ex:1} below.
In addition, the number of edges
and their total size can decrease; see Example~\ref{ex:2}.}

\begin{example}[Dimensions Increase]
\label{ex:1}
$$\begin{array}{ccc}
(1, 1) && (1,1) \\
 (1, 0)&& (0,0) \\
  (0, 1) &\Longrightarrow & (0,1) \\
   (2, 0) & \S_1 & (1,0)\\
   (0, 2) && (0,2)\\
   \end{array}$$
   Before shifting all three dimensions---Natarajan, DS and Pollard---are $1$.
   After shifting they are $2$.
   \end{example}
   
 \begin{example}[Edges Decrease]
 \label{ex:2}
 $$\begin{array}{ccc}
 (2, 2) && (0,2) \\
 (1, 1) & \Longrightarrow & (0,1)\\
 (1, 0) & \S_1 & (0,0)\\
 (2, 0) && (1,0) \\
 \end{array}$$
Before shifting, the three {non-singleton} edges are $\{(2, 2), (2, 0)\}$, $\{(1, 1), (1, 0)\}$, and $\{(1, 0), (2, 0)\}$, and the sum of their sizes is $6$. 
After shifting, there are two {non-singleton} edges $\{(0, 0), (0, 1), (0, 2)\}$ and $\{(0, 0), (1, 0)\}$,
and the sum of their sizes is $5$. 
\an{In the binary case,
the sum of the sizes of edges is equivalent to
the average degree, and it does not decrease during shifting.} 
\end{example}

These examples show that the analysis of multiclass shifting is 
not a direct extension of the arguments in the binary case.
We now identify two quantities that are similar to VC dimension and average degree, 
but can be controlled during shifting.

Because multiclass shifting is ``complex'',
we seek the simplest possible dimension so that we can keep
track of it.
\begin{definition} [Exponential Dimension]  \label{def:dexp}
We say that $S \in \X^n$ is $E$-shattered 
by $\H \subseteq \Y^\X$ if $\lvert \H|_S \rvert \geq 2^{n}$.
The exponential dimension $d_{E}(\H)$ is the maximum size
of an $E$-shattered sequence.
\end{definition}

The following claim shows that the exponential dimension 
is not increased during shifting.

\begin{claim} [Shifting Does Not Increase Projections]
Let $\H\subseteq [p]^n$ and let $i \in [n]$.
For every $S \in [n]^k$,
$$\big|\S_i(\H)|_S \big| \leq \big| \H|_S \big|.$$
\end{claim}
\begin{proof}
Without loss of generality, assume that 
$S$ has $k$ distinct entries.
Recall that 
$\H|_{S}$ is a subset of $[p]^{k}$.
We assume that $k>1$; the proof when $k=1$ is similar.
If~$i$ does not appear in $S$, shifting does not change the projection.
If~$i$ appears in $S$, argue as follows.
Let $S'$ be $S$ after deleting $i$, so
that $\H|_{S'}$ is a subset of $[p]^{k-1}$.
For each $f \in \H|_{S'}$,
let $e_f$ be the set of $h \in \H|_S$ that agree with $f$ on $[k-1]$,
so that
$$|\H|_{S}| = \sum_f |e_f|.$$
Let $e'_f$ be the set of $h \in \S_i(\H)|_S$ that agree with $f$ on $[k-1]$.
Because $\H|_{S'}$ is equal to $\S_i(\H)|_{S'}$,
we similarly have
$$|\S_i(\H)|_{S}| = \sum_f |e'_f|.$$
For fixed $f$,
\an{the size $|e'_f|$
is equal to the maximum size of an edge in the $i$'th direction of $\H$
whose elements agree with $f$.
This holds 
because there is an edge in the $i$'th direction of $\S_i(\H)$ of size $|e'_f|$
whose elements agree with $f$, and
the sizes of edges do not change during shifting.}
It follows that $|e_f| \geq |e'_f|$.

%
\end{proof}
\begin{corollary}
\label{cor:exp_dim_shifting}
For every $\H \subseteq [p]^n$ and $i \in [n]$,
$$d_E(\S_i(\H)) \leq d_E(\H).$$
\end{corollary}

We would like to control the 
{structure of edges during shifting.
The most obvious measure to keep track of is the average degree
(with respect to non-singleton edges).}

\begin{definition}[Average Degree]
Let $\G(\H) = (V,E)$ be the one-inclusion graph of $\H\subseteq [p]^n$. 
The average degree of $\H$ is
$$\ad(\H) = \frac{1}{|V|} \sum_{v \in V} \Deg(v)
= \frac{1}{|V|} \sum_{e \in E: |e|>1} |e|,$$
where $\Deg(v) = |\{e \in E : v \in e, |e|>1\}|$.
\end{definition}
Example~\ref{ex:2}
shows that the 
{number of edges and average degree may}
decrease during shifting (which is bad for our purposes).
{The correct measure to keep track of turns} out to be the following.

\begin{definition}[Shifting Average Degree] 
Let $\G(\H) = (V,E)$ be the one-inclusion graph of $\H\subseteq [p]^n$. 
Define 
$$\ad'(\H) = \frac{1}{|V|} \sum_{e \in E} (\vert e \vert - 1).$$
\end{definition}

\begin{claim}[Shifting Does Not Decrease $\ad'$]
For every $\H\subseteq [p]^n$ and $i \in [n]$,
$$\ad'(\S_i(\H)))\geq \ad'(\H).$$
\end{claim}
\begin{proof}
Since $\vert V \vert$ does not change during shifting, we only need to record the changes in the edges.
Let $i$ denote the direction of shifting.
The sizes of all edges in the $i$'th direction do not change during shifting.
We need to understand the behavior in the other directions. 
\an{We shall use the perspective explained after Definition~\ref{def:shift}.}
Fix $j \neq i$, and partition all edges in the $j$'th direction according
to their projection to $[n] \setminus \{i,j\}$.
Fix $f \in \H|_{[n] \setminus \{i,j\}}$,
and consider all vertices that agree with $f$ on $[n] \setminus \{i,j\}$.
Encode this data by the $p\times p$ Boolean matrix $M_{f}$
defined by $M_f(a,b) = 1$
iff adding $a,b$ to $f$ in positions $i,j$ leads to a word in $\H$.
\an{The example we saw earlier helps to digest
the scenario we are operating in (shifting six words
over an alphabet of size four):}
\begin{align*}
\left[ \begin{array}{cccc}
 1 & \ &1& \\
 & & & \\
 &&1& \\
1  &&1& 1\\
\end{array} \right]
\Longrightarrow 
\left[ \begin{array}{cccc}
 1 & \  &1& 1\\
1 & &1 & \\
 &&1& \\
  &&& \\
\end{array} \right]
\end{align*}
The sum of $|e|-1$ over all edges $e$ in the $j$'th direction that agree with $f$
can be expressed as the total number of $1$-entries in the matrix
minus the number of non-zero rows.
This is true for $\H$ as well for $\S_i(\H)$.
The total number of $1$-entries remains fixed.
The number of non-zero rows can not increase during shifting,
because the number of non-zero rows after shifting
is equal to the maximum number of $1$-entries in a single column.
It follows that this sum over $|e|-1$ can not decrease, as claimed.
%
%
%
%
%
\end{proof}

The control of the exponential dimension and of $\ad'$
allows to bound the average degree.

\begin{prop}[Average Degree is Bounded by Exponential Dimension]
\label{prop:mu<d}
For every $\H \subseteq [p]^n$,
$$\ad(\H) \leq 4 d_E(\H).$$
\end{prop}

\begin{proof}
Apply shifting on $\H$ until a fixed point $\H_*$ is reached.
Because
$$\ad(\H) \leq 2 \ad'(\H)
\leq 2 \ad'(\H_*) \leq 2 \ad(\H_*),$$
it suffices to prove that $\ad(\H_*) \leq 2 d_E(\H_*)$.
This holds by induction.
The induction base $|\H_*|=1$ trivially holds.
The induction step is justified as follows.
{Let $h_0$ be the concept in $\H_*$ so that
no concept in $\H_*$ is larger than $h_0$
with respect to the natural partial order. 
Let $|h_0|$
be the number of entries that are larger than $1$ in $h_0$.
The fixed point property of $\H_*$ implies that it is closed downwards.
There are $2^{|h_0|}$ words under $h_0$ in $\H_*$.
It follows that $\Deg(h_0) \leq |h_0| \leq d_E(\H_*)$.}
Remove $h_0$ from $\H_*$.
This removal does not increase the exponential dimension,
and the resulting class is still closed downwards.
The inductive assumption completes the proof,
because the sum
of the degrees is reduced by at most $2 \Deg(h_0)$.
\end{proof}

The bound on the average degree immediately 
allows to build good orientations.

\begin{corollary}[Exponential Dimension Leads to Orientations]
\label{cor:orient}
For every $\H \subseteq [p]^n$,
there is an orientation of $\G(\H)$ with maximum out-degree
at most~$4 d_E(\H)$.
\end{corollary}

\begin{proof}
Proposition~\ref{prop:mu<d} implies that
every induced sub-graph of $\G(\H)$ has a vertex of degree
at most $4d_E(\H)$.
The beginning of the proof of Lemma~\ref{lem:orientation}
produces the needed orientation;
the orientation is constructed ``greedily''
\an{by picking a vertex of degree at most $4 d_E(\H)$,
removing it from the graph and proceeding recursively.}
\end{proof}

The last piece of the puzzle is to relate the 
exponential dimension to the Natarajan dimension.
This is achieved via a generalization of Sauer's lemma
by~\citet*{HausslerL95}.

\begin{lemma}[Controlling the Exponential Dimension]
\label{lem:dedn}
For every $\H \subseteq [p]^n$ with $d_N = d_N(\H)$
and $d_E = d_E(\H) < \infty$,
\[d_E \leq  5 d_N \log(  p) .\]
\end{lemma}   

\begin{proof}
Corollary~5 in~\citep*{HausslerL95}
says that for every $S \subseteq [n]$ of size $m$,
the size of $\H|_S$ is at most
$$\sum_{i=1}^{d_N} {m \choose i}{p \choose 2}^i
\leq \Big( \frac{p^2 em}{2 d_N} \Big)^{d_N}.$$
\an{By definition,} the exponential dimension satisfies
$$2^{d_E} \leq \Big( \frac{p^2 e d_E}{2 d_N} \Big)^{d_N}.$$
This implies the lemma because $p \geq 2$.
\end{proof}

 \begin{remark}
 Corollary~\ref{cor:orient}
 and Lemma~\ref{lem:dedn}
 imply Lemma~\ref{lem:wl:rub}.
 \end{remark}

\section{Learnability $\equiv$ Finite DS Dimension}
\label{sec:ub}

Here we prove the characterization of multiclass PAC learnability via the DS dimension (Theorem~\ref{thm:ub}).
Our main contribution is algorithmic. 
We develop a learning algorithm
for any class $\H$ with finite DS dimension.

\subsection{Outline}\label{subsec:poverview}
    
    The starting point is 
    the \oig algorithm by \citet*{haussler1994predicting};
    see Section~\ref{subsec:oi} above for a reminder.  
    The finiteness of the DS dimension translates
    to a non-trivial guarantee \an{on} the \oig algorithm (as we saw in Claim~\ref{claim:good_pt}).
    The output hypothesis of this algorithm has expected prediction error at most $1-\frac{1}{d+1}$.
    This error is pretty high, but
    the crucial point is that it is uniformly bounded
    away from $1$.
\an{The \oig algorithm forms a kind of a (very) weak PAC learner.}
    
It is tempting to try to improve the error by boosting.
But standard boosting turns out to be useless in our context.
\an{The traditional assumption for boosting in the binary setting requires error below $\tfrac{1}{2}$.}
The above error guarantee is too weak and does not meet the minimal requirements for boosting.
And even if multiclass boosting \an{was} available,
known techniques have sample complexity 
that scales with $\log \lvert \Y\rvert$; see~\citep*{schapire2012boosting,brukhim2021multiclass}.
This factor could be infinite in \an{our} setting.
\an{To circumvent this obstacle we introduce the framework of list PAC learning.}

\subsection*{List PAC learning}
In the standard PAC setting,
the goal is to provide a single prediction on an unseen data point.
In list PAC learning,
the goal is to provide a short menu of predictions.
Given a sample $S \sim \D^n$ from a realizable $\D$,
the goal is to output a menu $\mu$ that maps elements of $\X$
to a small subset of $\Y$
so that $y \in \mu(x)$ with high probability 
over a new test point $(x,y) \sim \D$.
List PAC learning is discussed in greater detail in Section~\ref{sec:listpac}.

%

\an{Rather than boosting the weak \oig algorithm to a strong PAC learner, 
    we use it to derive a list PAC learning algorithm.
We show that every class $\H$
    with a finite DS dimension admits a list PAC learner}
(see Algorithm~\ref{algo:list-pac}).
This list-learner gathers information from several \oig algorithms
to produce a good menu.
\an{Its analysis is based on the leave-one-out symmetrization argument}.
The list-learner allows to eliminate the vast majority
of a priori possible labels.
Instead of all of $\Y$, we can safely use the menu $\mu(x)$
as the ``local alphabet for $x$''.
\an{Menus can be thought of as tools for alphabet reduction.
Once we have a list PAC learner, it is natural to try to reduce 
the learning task to one in which the number of labels is bounded.}

Did we just reduce the infinite alphabet case to the finite case?
The short answer is no.
Even with a good menu $\mu$,
the subclass
$\H\vert_\mu=\{h\in \H : \forall x \in \X \  h(x)\in \mu(x)\}$
of $\H$ may be completely useless.
For example, let $\H\subseteq\{0,1,2\}^\mathbb{N}$ be the set of all functions $h$ such that $\lvert h^{-1}(\{1,2\})\rvert < \infty$, and let $\mu$ be the menu
such that $\mu(x)=\{1,2\}$ for all $x\in\mathbb{N}$. 
The menu-subclass $\H\vert_\mu$ is just empty.
At the same time,
every finitely supported distribution $\D$ with labels in $\{1,2\}$ 
    is both realizable by~$\H$ and consistent with~$\mu$.
    This simple example indicates that 
    in order to restrict to a subclass of $\H$ without losing essential information,
    at least some knowledge on the support of the target distribution is needed.
    Learning the support of a distribution, however,
    is a much harder task than PAC learning.

\an{Let us make a quick comment.
In this work, list PAC learning serves as a tool for proving Theorem~\ref{thm:ub}.
However, we think it is a natural setting and interesting in its own right
(and we prove further motivation in Section~\ref{sec:listpac}).}

 
\subsection*{List PAC Learning $\Rightarrow$ PAC Learning} 
Our \an{solution is based on the fact} that 
the \oig algorithm is exactly suitable for situations
in which the learning task is not defined by a concept class, 
but by a set of allowable samples.
The main property of \oig algorithms is their {\em locality}.
To make a prediction on $x$, 
they just use the part of $\H$
that is relevant to the training data $S$,
and do not require any global access to $\H$.

An alternative way to model learning with a menu 
    is via partial concept classes~\citep*{alon2021theory}.
    Instead of all maps in $\H$ that are consistent with the menu $\mu$,
    we can consider all partial maps
    that agree with both the class $\H$ and the menu $\mu$. 
    We chose not to use this formalism \an{here}
    in order to use as standard language as possible.
 The partial concept class perspective does not really
 help to solve the problem.   
    The focus of \citet*{alon2021theory} was on binary-classification, 
    which is significantly simpler than the multiclass setting.
    Generalizing the analysis of the one-inclusion graph
    from the binary setting to the multiclass setting turns out to be a subtle
    (and somewhat confusing) task.
    Natural attempts to do so fail;
    see Section~\ref{sec:shift} for a full discussion.

\begin{figure}
\centering
\includegraphics[width = 13cm]{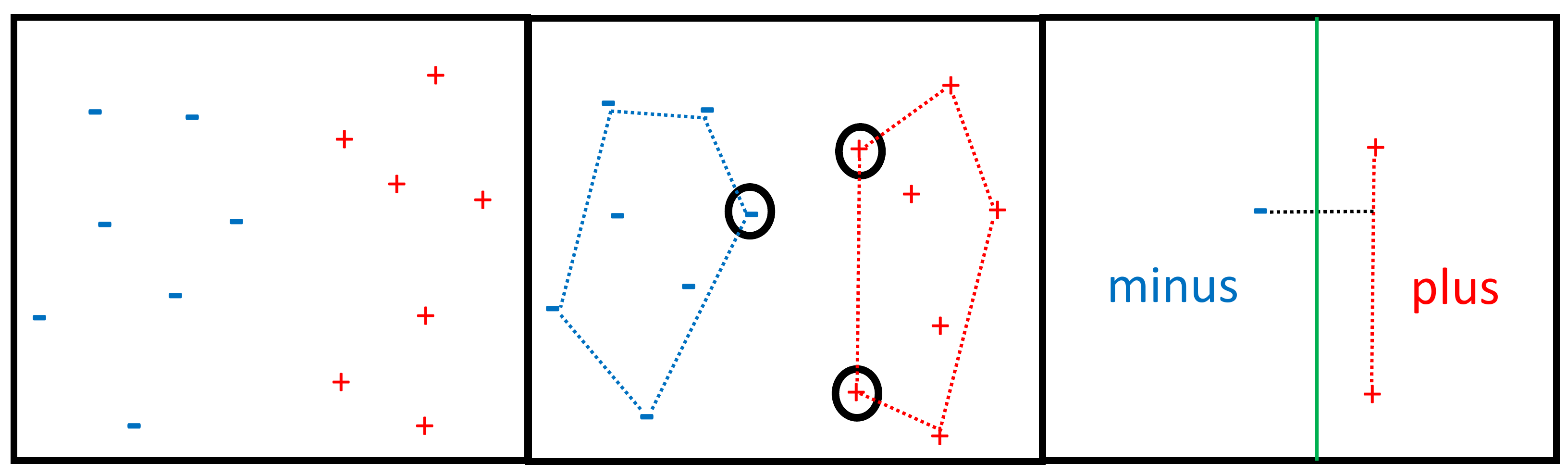}
\caption{An illustration of Support Vector Machine in 2D as a sample compression
scheme.
The realizable sample $S$ consists of negative and positive points.
The algorithm outputs a separating line that maximizes the margin. 
This line is determined by the support vectors
(circled).}\label{fig:svmcomp}
\end{figure}

\subsection*{Sample Compression Schemes}
The algorithm we develop is best thought of as a {\em sample compression scheme}~\citep*{littlestone1986relating}.
A sample compression scheme (Definition~\ref{def:scs}) is an abstraction
of a common property to many learning algorithms;
see Figure~\ref{fig:svmcomp}.
It can be viewed as a two-party protocol 
    between a {\it compresser} and a {\it reconstructor}.
    Both players know the underlying concept class $\H$.
The compresser gets as input an $\H$-realizable sample $S$.    
The compresser picks a small subsample $S'$ of $S$ and sends it to the reconstructor.
The reconstructor outputs an hypothesis~$h$.
    The correctness criteria is that $h$ needs to 
    correctly classify \emph{all} examples in the input sample $S$.

    One advantage of using the sample compression schemes framework
    is that the proofs are typically cleaner, because
    in contrast to the probabilistic nature of the PAC framework,
    sample compression is a deterministic task.
    At the same time, sample compression schemes are known
    to represent good PAC learning algorithms~\citep*{littlestone1986relating}.

Classical sample compression algorithms usually boil down to
a simple one-shot encoding scheme (e.g.\ Figure~\ref{fig:svmcomp}).
Our compression scheme is more involved and is comprised
of two main components.
The first component is a variant of sample compression
that fits into the list-learning framework (Definition~\ref{def:lscs}).
The second component incorporates the menu derived by the first component 
together with a minimax-based sample compression 
as in~\citep*{david2016supervised}.
All of this is described in Section~\ref{subsec:cs}.

\an{A high-level outline of the full algorithm is presented in Figure~\ref{fig:outline},
after all the needed ingredients are prepared
and the subtleties are discussed.}





\subsection{List PAC learning}\label{sec:listpac}
List PAC learning is a model 
for providing a short menu of likely predictions.
It extends the standard PAC model
by allowing the learning algorithm more freedom.

\an{Relaxing the demand of a single output to a list of outputs
is a common and useful paradigm in computer science.
One notable example is the notion of {\em list-decoding}
in coding theory, which is important both as a tool and as a goal.}

Let us start with a few examples \an{for list learning}.
{In medical contexts, list-learning can offer physicians 
a menu of likely diagnoses.
In technical contexts, list-learning can 
provide a short menu of possible solutions that are meant to assist
clients.
List-learning can also provide
the menu of preferences of consumers.
One can easily imagine other scenarios where list-learning is useful.}

Our main motivation for developing this model
is reasoning on the first component of our multiclass learning algorithm.
But this basic model naturally fits into many scenarios,
and we plan to investigate it further in future works.

The goal of list PAC learning is to compute good menus.

\begin{definition} [$p$-menu]\label{def:dict}
A menu of size $p \in \mathbb{N}$ is a function $\mu:  \X \rightarrow  \{Y \subseteq\Y : \lvert Y\rvert\leq p\}$.
 \end{definition}

List PAC learning is the following natural version
of standard PAC learning.

\begin{definition}[List PAC Learner]
    An algorithm $A$ with sample size $n$ and list size $p$
    is a {\em list PAC learner} with success probability $\alpha >0$
    for the concept class $\H \subseteq \Y^\X$ if
    for every $\H$-realizable distribution $\D$,
\[\Pr_{(S,(x,y)) \sim \D^{n+1}}\bigl[y\in \mu_S(x)\bigr]\geq \alpha,\]
    where $\mu_S = A(S)$ is always a $p$-menu.
\end{definition}

\begin{remark}
\an{In the ``noisy'' case, when
the label $y$ has entropy given $x$,
list learning can potentially lead to zero error
even though in the standard PAC setting
zero error is not achievable.
}
\end{remark}

The main result of this section is the development of a
list PAC learner for every class of finite DS dimension
(see Algorithm \ref{algo:list-pac}).
The list PAC learner can be thought of as a brute-force
\an{extension} of the one-inclusion learning rule.

\begin{algorithm}
\caption{List PAC learner $\mathcal{L}_{\H,t}$ for $\H \subseteq \Y^\X$ with $d_{DS}(\H) =d$ and $t \in \mathbb{N}$} \label{algo:list-pac}
\begin{flushleft}
  {\bf Input:} Data $S \in (\X \times \Y)^n$ where $n=d+t$.  \\
{\bf Output:} A $p$-menu $\mu_S$ for $p = {n \choose t}$. \\
\end{flushleft}
\begin{algorithmic}[1]
\STATE Let $S_1,\ldots,S_p$ denote all subsamples of $S$ of size $d$.
\STATE Let $h_{S_i}=\A_{\H}(S_i)$ denote the hypothesis \an{output of} Algorithm \ref{algo:one_inc} on input sample $S_i$.
\STATE Return the menu defined by
\[
\mu_S(x) = \bigl\{h_{S_1}(x),\ldots, h_{S_p}(x)\bigr\}.
\]
\end{algorithmic}
\end{algorithm}

\begin{prop}[Finite DS Dimension implies List PAC Learning]  \label{prop:weak-list-pac}
Let $\H \subseteq \Y^\X$ be a class with DS dimension $d<\infty$ and let $t\in\mathbb{N}$.
The algorithm $\mathcal{L}_{\H,t}$ is a
list PAC learner for $\H$ with sample size $n = d+t$,
list size $p = {n \choose t}$ and success probability $\alpha=\tfrac{t+1}{d+t+1}$.
\end{prop} 

\begin{proof}
Let $\mu_S = \mathcal{L}_{\H,t}(S)$ be the menu generated by the algorithm
with data $S$.
By the leave-one-out symmetrization argument (Fact \ref{fact:sym}), 
\begin{align*}
\Pr_{(S,(x,y)) \sim \D^{n+1}}\bigl[y\in \mu_S(x)\bigr] &= 
\Pr_{(S',I) \sim \D^{n+1} \times \textit{U}([n+1]])}\bigl[y'_I\in \mu_{S'_{-I}}(x'_I)\bigr]. 
\end{align*}
It hence suffices to show that every realizable sample $S'$ of size $n+1$ satisfies
\begin{equation} \label{eq:td_}
    \Pr_{I \sim \textit{U}([n+1]])}\bigl[y'_I\in \mu_{S'_{-I}}(x'_I )\bigr]
\ge \frac{t+1}{d+t+1}.
\end{equation}
Let us call an index~$i\in [n+1]$ {\it good} if $y'_i\in \mu_{S'_{-i}}(x'_i)$. 
We need to show that there are at least $t+1$ good indices. 
By Claim~\ref{claim:good_pt}, 
at least one of the indices in $[d+1]$ is good.
Denote this good index by $i_1$. 
Again, by Claim~\ref{claim:good_pt}, 
at least one of the indices in $[d+2] \setminus \{i_1\}$ is good.
Denote this good index by $i_2$. 
Repeat this process to obtain the needed $t+1$ good indices.
\end{proof}

\subsection{Learning Natarajan Classes From Menus} \label{subsec:img}
We now move towards the second component of our algorithm.
The objective is to use the good menu that was generated
by the first component to effectively reduce the number of labels.
The algorithm we develop in this sub-section is a weak PAC learner,
but under a strong assumption.
Several such weak learners will be combined later on
to get the full sample compression scheme.

The learning algorithm now has two pieces of knowledge:
the underlying class $\H$ and the menu $\mu$.
Trusting that the first component delivered on its promise,
it assumes that the data is consistent with the menu.
This is captured by the following definition.

\begin{definition}[Menu Realizability]
A sample $S \in (\X \times \Y)^n$
is \emph{realizable} by the menu $\mu$ if $y\in \mu(x)$ for every $(x,y)$ in $S$. 
A distribution $\D$ over $\X \times \Y$ is {\em realizable} by $\mu$ if for every $m \in \mathbb{N}$,
a random sample $S\sim \D^m$ is realizable by $\mu$ with probability~$1$. 
\end{definition}

\an{This definition captures the ideal scenario
that we have a menu that completely captures 
the unknown distribution $\D$.
It is basically impossible to generate a single menu
that captures all of $\D$.
Nevertheless, this idealization is a useful sub-goal
that we need to deal with later on.}

\begin{algorithm}
\caption{One-inclusion algorithm $\A_{\H,\mu}$ 
for a class $\ch$ and menu $\mu$} 
\label{algo:weak-pac}
\begin{flushleft}
  {\bf Input:} A sample $S = \big((x_1, y_1),\ldots,(x_n, y_n)\big)$ realizable by $\H$ and $\mu$. \\
{\bf Output:} A hypothesis $h_{S}: \X \to \Y$. \\
\ \\
For each $x \in \X$, the value $h_S(x)$ is computed as follows.
\end{flushleft}
\begin{algorithmic}[1]
\STATE Consider the class $\H' \subseteq \Y^{n+1}$ of all patterns on 
the unlabelled data that are realizable by both $\H$ and $\mu$.
That is, it is the set of all {$h \in \H|_{(x_1,\ldots,x_n,x)}$
so that $h(n+1) \in \mu(x)$ and $h(i) \in \mu(x_i)$ for $i \in [n]$.}
\STATE Find an orientation $\sigma$ of $\mathcal{G}(\H')$ that minimizes the maximum out-degree.
\STATE Consider the edge in direction $n+1$
that is consistent with $S$. Let
\[e=\bigl\{ h \in \H' :  \forall i \in [n] \ h(i) =y_i \bigr\}.\]
\STATE Let $h' = \sigma((e,n+1))$.
\STATE Set $h_{S}(x)=h'(n+1)$.
\end{algorithmic}
\end{algorithm}

The main result of this sub-section is a PAC learning algorithm
for menu-realizable distributions
(Algorithm~\ref{algo:weak-pac}).
The sample complexity is controlled by the
size of the menu $\mu$ as well the Natarajan dimension of $\H$.
This is pretty good news because
we controlled the size of the menu, and
the Natarajan dimension is the smallest among all dimensions.

\begin{prop}[PAC Learning Given a Menu]  \label{prop:weak-learner}
Let $\H\subseteq \Y^\X$ be a class with Natarajan dimension $d_N <\infty$ and let $\mu$ be a~$p$-menu. 
For every distribution $\D$ over $\X \times \Y$ that is realizable by both $\H$ and by~$\mu$, and for all integers $n>0$,
\[\Pr_{(S,(x,y))\sim \D^{n+1}}\bigl[h_S(x)\neq y\bigr]\leq \frac{20 d_N \log(p)}{n},\]
where $h_S = \A_{\H,\mu}(S)$.
\end{prop}  

The proposition is not the end of the story.
The menu $\mu$ generated by the first component
allows to make good list-predictions,
but it has no chance to capture all of the unknown distribution $\D$.
It is extremely unlikely that $\D$ is realizable by $\mu$.
The removal of this realizability assumption is postponed to the next section.

The high-level idea behind the proof 
of the proposition is to use the $p$-menu to reduce 
the label-set from the unbounded $\Y$
to a label-set of size $p$. 
This is beneficial because PAC learning with $p$ many labels
can be achieved with sample complexity
order $d_N\log(p)$. In fact, any proper ERM algorithm with this sample complexity is a PAC learner.

Trying to implement this strategy raises a subtle challenge.
The task of learning a distribution realizable by a class $\H$ and a menu $\mu$ 
\emph{cannot} be reduced to PAC learning the sub-class $\H|_\mu$ of $\H$ 
that is consistent with the menu.
The sub-class $\H|_\mu$ could even be empty;
see Section~\ref{subsec:poverview} for a simple example.
    
The solution is based on the unique locality feature of
the \oig algorithm.
To make a prediction on an unseen data point $x \in \X$,
the \oig algorithm
just uses $x$ and the unlabelled part of the sample~$S$.
This local view of $\X$ suffices to make a prediction.

\begin{proof}[Proof of Proposition~\ref{prop:weak-learner}] 
Let $\D$ be a distribution that is realizable by $\H$ and $\mu$.
By the leave-one-out symmetrization argument (Fact~\ref{fact:sym}),
\begin{align*}
\Pr_{(S, (x,y))\sim \D^{n+1} }\bigl[h_{S}(x) \neq y \bigr] 
= \Pr_{(S',I) \sim \D^{n+1} \times \textit{U}(n+1)}\bigl[ h_{S'_{-I}}(x'_I)  \neq y'_I\bigr] ,
\end{align*}
where $h_S = \A_{\H,\mu}(S)$.
It therefore suffices to show that for every sample $S'$ 
that is realizable by $\H$ and $\mu$,
\begin{equation} \label{eq:wl_}
    \Pr_{I \sim \textit{U}(n+1)}\bigl[  h_{S'_{-I}}(x'_I)  \neq y'_I \bigr]
\le \frac{20d_N\log(p)}{n}.
\end{equation}
Fix $S'$ that is realizable by $\H$ and $\mu$ for the rest of the proof.
The class {$\H' = \H|_{(x'_1,\ldots,x'_{n+1})}$} constructed by the algorithm $\A_{\H,\mu}$
for $S'_{-I}$ and $x'_I$ is the same for all values of~$I$.
The class $\H'$ is realizable by $\mu$.
The Natarajan dimension of $\H'$ is at most that of $\H$.
Denote by $\sigma$ the orientation of $\G(\H')$
that the algorithm chooses.
Lemma~\ref{lem:wl:rub} tells us that the maximum out-degree
of $\sigma$ is at most $20 d_N \log(p)$.
Let~$y'$ denote the vertex in $\G(\H')$ defined by $y'_1,\ldots,y'_{n+1}$.
Finally,
\begin{equation*} \label{eq:bound_img}
     \Pr_{I \sim \textit{U}(n+1)}\bigl[  h_{S'_{-I}}(x'_I)  \neq y'_i\bigr]
    =    \frac{\dout(y' ; \sigma)}{n+1} 
\le \frac{20 d_N \log(p)  }{n+1} .\qedhere
\end{equation*} 
\end{proof}

\subsection{The Algorithm} \label{subsec:cs}
We are ready to describe the full algorithm.
It is convenient to think of it as a sample compression scheme.
    
\begin{definition}[Sample Compression Scheme~\citep*{littlestone1986relating}]\label{def:scs}
Let $\H \subseteq \Y^\X$ and let $r\leq n$ be integers. 
    An $n\to r$ \emph{sample compression scheme} 
    consists of a \emph{reconstruction} function
    \[\rho:(\X\times\Y)^r\to \Y^\X\]
    such that for every $\H$-realizable $S\in (\X\times \Y)^n$,
    there exists $S'\in (\X\times \Y)^r$ whose elements appear in $S$
    such that for every $(x,y)$ in $S$ we have
    $h(x)=y$, where $h=\rho(S')$.
\end{definition}

The main goal of this section is to construct a sample compression scheme for 
classes with finite DS dimension.

\begin{theorem}[DS Classes are Compressible]\label{thm:compressionfinal}
Let $\H \subseteq \Y^\X$ be a class with DS dimension $d_{DS}<\infty$
and Natarajan dimension $d_N$.
For every integers $n,t>0$,
there exists an $n\to r$ sample compression scheme for $\H$ with
\begin{align*}
r & \leq  \Bigg(\frac{d_{DS}+t+1}{t+1} (d_{DS}+t)  + 
10^3 d_N\log\Bigg({{d_{DS} + t+1}  \choose t+1} \log(2n)\Bigg)
\Bigg) \log(2n).
\end{align*}
\end{theorem}

When $d_{DS}$ and $t$ are thought of as constants and $n$ as tending
to infinity,
the value of $r$ becomes
$r \leq \tilde O ( (\tfrac{(d_{DS}+t)^2}{t} + d_{DS} t ) \log n)$.
For $t = \lceil d_{DS}^{1/2} \rceil$, this becomes 
$$r \leq \tilde O ( d_{DS}^{3/2} \log n).$$
As we explain in the next section,
a standard 
``compression implies generalization'' 
argument implies that $\tilde O( d_{DS}^{3/2})$ samples are sufficient for PAC learning.

The sample compression scheme consists of
two components.
The first component 
provides list-learning guarantees.
It produces a good menu 
that is passed to the second component. 
The second component is a regular sample compression scheme
but only for menu-realizable samples.
To accommodate this mechanism,
we introduce two variants of sample compression schemes.

\begin{definition}[List Sample Compression Scheme]\label{def:lscs}
    An $n\to r$ \emph{list sample compression scheme
    with menu size~$p$}
consists of {a \emph{reconstruction function} 
    \[\rho:(\X\times\Y)^r\to \{Y \subseteq \Y : |Y| \leq p\}^\X\] 
    such} that for every $\H$-realizable $S\in (\X\times \Y)^n$,    
    there exists $S'\in (\X\times \Y)^r$ 
    whose elements appear in $S$
   such that for every $(x,y)$ in $S$ we have
    $y \in \mu(x)$, where $\mu=\rho(S')$.
\end{definition}

\begin{definition}[Sample Compression Scheme for a Menu]\label{def:lscs}
    An $n\to r$ \emph{sample compression scheme} for a class $\H$
    and a menu $\mu$ 
consists of a \emph{reconstruction function} 
    \[\rho:(\X\times\Y)^r\to \Y^\X\] 
    such that for every $S\in (\X\times \Y)^n$
    that is realizable by both $\H$ and $\mu$,    
    there exists $S'\in (\X\times \Y)^r$ 
    whose elements appear in $S$
   such that for every $(x,y)$ in $S$ we have
    $h(x) = y$, where $h=\rho(S')$.
\end{definition}

The following two lemmas summarize the two components
of the construction.

\begin{lemma}[List Sample Compression Scheme]\label{lem:listcompression}
Let $\H \subseteq \Y^\X$ be a class with DS dimension $d_{DS}<\infty$.
For every integers $n,t>0$, there exists an $n\to r_1$ list sample compression scheme for $\H$ with menu size $p$, where
$$r_1 \leq \frac{d_{DS}+t+1}{t+1} (d_{DS}+t) \log(2n)$$
and
$$p \leq  {{d_{DS} + t +1}  \choose t+1} \log(2n).$$
\end{lemma}

\begin{lemma}[Sample Compression Given a Menu]\label{lem:compressionfromdict}
Let $\H \subseteq \Y^\X$ be a class with Natarajan dimension $d_N<\infty$
and let $\mu$ be a $p$-menu.
For every integer $n>0$, there exists an $n\to r_2$ sample compression scheme 
for $\H$ and $\mu$ with 
$$r_2 \leq  10^3 d_N\log(p)\log (2n).$$ 
\end{lemma}


The two lemmas, which are proved below, complete the proof
of Theorem~\ref{thm:compressionfinal}.

\begin{proof}[Proof of Theorem~\ref{thm:compressionfinal}]
The high-level outline is presented in Figure~\ref{fig:outline}.
Let $S$ be an $\H$-realizable sample of size $n$.
Lemma~\ref{lem:listcompression} 
tells us that there is a reconstruction $\rho_1$ that produces
$p$-menus,
and a sequence $S'$ of $r_1$ examples from $S$ such that 
$S$ is $\mu$-realizable where $\mu = \rho_1(S')$.
Lemma~\ref{lem:compressionfromdict} applied to $\H$ and $\mu$
implies that there is a reconstruction $\rho_2$, and
a sequence $S''$ of $r_2$ examples from $S$ such that
$\rho_2(S'')$ correctly classifies the entire sample~$S$.
The composition of the two schemes
is an $n\to r_1+r_2$ sample compression scheme for~$\H$.
\end{proof}

\begin{figure}

\noindent\makebox[\linewidth]{\rule{\textwidth}{0.8pt}}

{\bf High-level Outline of the Algorithm}

\medskip

realizable case over $\H$ with $d = d_{DS}(\H)< \infty$

\noindent\makebox[\linewidth]{\rule{\textwidth}{0.4pt}}

\begin{flushleft}
{\bf Input:} A sample $S \in (\X \times \Y)^n$.
\end{flushleft}

\begin{algorithmic}[1]
\STATE Find $m' \approx \sqrt{d} \log n$ subsamples $S'_1,\ldots,S'_{m'}$
of $S$, each of size $d$, so that 
the menu~$\mu$ of size $p \approx 2^{\sqrt{d}} \cdot m'$
they define realizes $S$.

\STATE Using the menu $\mu$, find $m'' \approx \log n$ subsamples $S''_1,\ldots,S''_{m''}$
of $S$, each of size $\approx d^{3/2}$, so that 
the majority vote $h$ over the $m''$ functions they define
correctly classifies $S$.
\end{algorithmic}

\begin{flushleft}
{\bf Output:} The function $h:\X \to \Y$.
\end{flushleft}

\noindent\makebox[\linewidth]{\rule{\textwidth}{0.4pt}}

\medskip

\begin{tikzpicture}[node distance=2cm]
\tikzstyle{startstop} = [rectangle, rounded corners, minimum width=2.5cm, minimum height=1cm,text centered, draw=black, fill=cyan!10]
\tikzstyle{arrow} = [thick,->,>=stealth]

\node (s1) [startstop] {input $S$};
\node (sp1) [startstop, below of=s1] {subsample $S'$};
\node (mu1) [startstop, below of=sp1] {menu $\mu$};
\node (spp1) [startstop, right of=mu1,node distance=4cm] {subsample $S''$};
\node (h1) [startstop, below of=spp1] {output $h$};

\draw [arrow] (s1) -- (sp1);
\draw [arrow] (s1)  -| (spp1);
\draw [arrow] (sp1) -- node[anchor=west] {$\rho_1$} (mu1) ;
\draw [arrow] (mu1) -- (spp1);
\draw [arrow] (spp1) -- node[anchor=west] {$\rho_2$} (h1) ;
\draw [arrow] (mu1)  |- (h1);

\end{tikzpicture}
\caption{The outline of the algorithm.
The $m'$ subsamples in step one are found using the
\oig algorithm of the class $\H$.
The $m''$ subsamples in step two are found using the
\oig algorithm of the class $\H$ and the menu $\mu$.}
\label{fig:outline}
\end{figure}
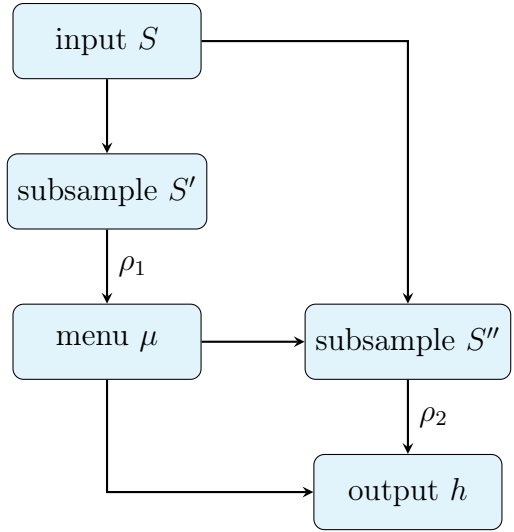

\subsubsection*{The List Compression Scheme}

\begin{proof}[Proof of Lemma~\ref{lem:listcompression}]
We begin by describing the reconstruction function $\rho_1$.
Let $\ell=\lfloor \frac{d_{DS}+t+1}{t+1} \log(2n)\rfloor$.
Given an $\H$-realizable sample $S'$ of size $r_1 = (d_{DS}+t) \ell$,
    partition it into $\ell$
    contiguous subsequences~$S'_1,\ldots, S'_\ell$,
    each of {size $d_{DS}+t$.} Define $\mu = \rho(S')$ as
    \[ \mu(x)=\bigcup_{j=1}^\ell \mu_j(x),\]
    where $\mu_j(x)=\mathcal{L}_{\H}(S'_j)$ is the ${d_{DS}+t \choose t}$-menu outputted
    by Algorithm~\ref{algo:list-pac} on input sample~$S'_j$.
    The menu $\mu$ has list-size $p\leq  {d_{DS}+t+1 \choose t+1} \log(2n)$.

It remains to show that there is $S'$ so that
the reconstruction on $S'$ achieves its goal.
The sample $S'$ is chosen via the probabilistic method.
%
%
Let $U$ denote the uniform distribution over the $n$ examples
in $S$ and let $\alpha=\frac{t+1}{d_{DS}+t+1}$. 
By Proposition~\ref{prop:weak-list-pac} applied to the distribution $U$, 
for a random sample~$S'_1 \sim U^{d_{DS}+t}$, 
in expectation at least~$\alpha n$ of the examples $(x_i,y_i)$ in $S$ satisfy
    \[y_i\in \mu_1(x_i),\]
    where $\mu_1=\mathcal{L}_{\H,t}(S'_1)$.
    In particular, there exists $S'_1$ for which the above holds.
Remove from $S$ all examples $(x_i,y_i)$ such that $y_i\in \mu_1(x_i)$ and repeat the same reasoning on the remaining sample.
    This way at each step $j$ we find a sample $S'_j$ and a menu $\mu_j=\mathcal{L}_{\H,t}(S'_j)$ 
    that covers at least an $\alpha$-fraction of the remaining examples. 
    After $\ell$ steps,
all examples in $S$ are covered because $(1-\alpha)^{\ell}n < 1$.
    Setting $S'$ to be the concatenation of $S'_1,\ldots ,S'_\ell$ finishes the proof.
    \end{proof}

\subsubsection*{Learning From a Menu}
\begin{proof}[Proof of Lemma~\ref{lem:compressionfromdict}]
We begin by describing the reconstruction function $\rho_2$.
Let $\ell= \lfloor 8\log(2n) \rfloor$ and
$m=\lceil 100 d_N\log (p) \rceil$.
    Given a sequence of $r_2= m \ell$ examples $S'$ that are realizable by $\H$ and $\mu$,
    partition it into $\ell$ 
    contiguous sub-sequences~$S'_1,\ldots, S'_\ell$,
    each of size $m$. 
    Define $h(x) = \rho_2(S')$ as
    \[ h(x)=\mathsf{plurality}\bigl(h_1(x),\ldots, h_\ell(x)\bigr)\]
    where $h_j=\A_{\H,\mu}(S'_j)$ is the hypothesis outputted
    by Algorithm~\ref{algo:weak-pac} on input sample $S'_j$,
    and $\mathsf{plurality}(y_1,\ldots, y_\ell)$ is the label that appears
    most frequently (breaking ties arbitrarily).

It remains to explain how to choose $S'$ 
from a sample $S$ that is realizable by 
    $\H$ and $\mu$.
    The existence of $S'$ follows from the probabilistic method.
This time we also rely on von Neumann's minimax theorem~\citep*{neumann1928theorie}. 

We first claim that there exists a distribution $\mathcal{P}$ over sequences $T$ of size $m$ with elements from~$S$
    such that for every example $(x,y)$ in $S$,
    \begin{equation}\label{eq:minimax}
    \Pr_{T\sim \mathcal{P}}\bigl[h_{T}(x)\neq y\bigr]\leq \frac{1}{4},
    \end{equation}
    where $h_T = \A_{\H,\mu}(T)$.
 Consider a zero-sum game between two players Minnie and Max.
    Max's pure strategies are examples $(x,y)$ in $S$.
    Minnie's pure strategies are sequences~$T$ of $m$ elements 
    from $S$.
    The payoff matrix $L$ is defined by
    $L_{(x,y),T} = 1_{h_{T}(x)\neq y}$.
    Let $\mathcal{Q}$ be a mixed strategy of Max.
    Namely, $\mathcal{Q}$ is a distribution over examples in $S$.
The distribution $\mathcal{Q}$ is realizable by both $\H$ and $\mu$.
Proposition~\ref{prop:weak-learner} implies
    \[\Pr_{(T,(x,y))\sim \mathcal{Q}^{m+1}}\bigl[h_{T}(x) \neq y \bigr]\leq \frac{20 d_N \log (p)}{m} \leq \frac{1}{4}.\]
    In words, for every mixed strategy of Max,
    there is a strategy for Minnie with cost at most $\tfrac{1}{4}$.
    By the minimax theorem,
    there is a mixed strategy for Minnie that
    guarantees cost at most $\tfrac{1}{4}$ for
    every strategy of Max.
    This mixed strategy is the required $\mathcal{P}$.
    
%
%
    
 The existence of $S'$ can finally be proved.
 Let $S'_1,\ldots ,S'_\ell$ be i.i.d.\ samples from~$\mathcal{P}$.
Standard concentration of measure implies that for each $(x,y)\in S$,
    \[\Pr \Bigl[\frac{1}{\ell}\sum_{j=1}^\ell1_{h_{S'_j}(x)\neq y} \geq \frac{1}{2}  \Bigr] \leq \exp\Bigl(- \frac{\ell}{8}\Bigr) <\frac{1}{n}.\]
The union bound implies that with positive probability,
for every $(x,y)$ in $S$ we have $\frac{1}{\ell}\sum_{j=1}^\ell1_{h_{S'_j}(x)\neq y} < \frac{1}{2}$.
In particular, there exist $S'_1,\ldots, S'_\ell$
such that the plurality vote
over the $h_{S'_1},\ldots,h_{S'_\ell}$ correctly classifies all of $S$.
The concatenation is the required $S'$.
        \end{proof}

\subsection{Wrapping-up}\label{sec:wrap}
\begin{proof}[Proof of Theorem \ref{thm:ub}]
Theorem~\ref{thm:compressionfinal} with $t = \lceil d^{1/2} \rceil$
states the existence of an $n\to r$ sample compression scheme for $\H$ 
where $r \leq O( d^{3/2} \log (n))$.
%
The analysis of the sample compression scheme 
relies on previous results on multiclass compression.
Theorems 3.1 and 3.3 in \citep*{david2016supervised} imply that if $\H$
    admits an $n\to r$ sample compression scheme, then the same compression scheme is a learning rule $A^{real}$ with the following guarantee.
    For every $\H$-realizable distribution $\D$, every $\delta > 0$
    and every integer $n>0$,    
    with probability at least $1-\delta$ over $S\sim \D^n$, 
    the output hypothesis $h=A^{real}(S)$ satisfies
    \begin{equation}\label{eq:realizablerate}
      \Pr_{(x,y)\sim \D}\bigl[h(x)\neq y\bigr] \leq O\Biggl(\frac{r \log\bigl(\frac{n}{r}\bigr) + \log(1/\delta)}{n}\Biggr).  
    \end{equation}
    In the agnostic case, they prove the existence of a related learning rule $A^{agn}$
    with the following guarantee.
    For every distribution $\D$, every $\delta > 0$,
    and every integer $n>0$,
    with probability at least $1-\delta$ over sampling $S\sim \D^n$, 
    the output hypothesis $h=A^{agn}(S)$ satisfies
    \begin{equation}\label{eq:agnosticrate}
      \Pr_{(x,y)\sim \D}\bigl[h(x)\neq y\bigr] \leq L_{\D}(\H)  + O\Biggl(\sqrt{\frac{r\log\bigl(\frac{n}{r}\bigr) + \log(1/\delta)}{n}}\Biggr).  \qedhere
    \end{equation}
\end{proof}

\section{Learnability $\not\equiv$ Finite Natarajan Dimension}\label{sec:lb}

The goal of this section is to prove that the Natarajan dimension does not characterize PAC learnability (Theorems~\ref{atm:lb} and~\ref{thm:lb}).
That is, to construct a concept class $\H$
that has Natarajan dimension $1$
but DS dimension $\infty$.

\subsection{Outline}

The class $\H$ lives between two opposing conditions.
On one hand,
there should be no non-trivial copy
of the Boolean cube inside $\H$.
On the other hand, 
it should contain pseudo-cubes of arbitrary large dimension.
{Pseudo-cubes of large dimension 
imply that learning $\H$ is difficult.
No copies of the Boolean cube
indicates that ``locally $\H$ looks like it is easy to learn''.
The barrier \an{to} learning $\H$ is not local
but global.
An analogy is a graph of large girth and large chromatic number;
locally the graph is $2$-colorable,
and the coloring-difficulty stems from a global obstacle.}

Our goal is, essentially, to find pseudo-cubes of arbitrary large dimension
that do not contain any non-trivial copy of the Boolean cube.
The proof begins by translating the problem from the realm of concept classes
to the realm of simplicial complexes
(see Section~\ref{sec:PCvsSC}).
We show that any concept class can be identified with a colorful simplicial 
complex (and vice versa).

What about the pseudo-cube condition and the Natarajan dimension
in the realm of simplicial complexes?
The pseudo-cube conditions turns out to be quite natural; it is reminiscent \an{of} the notion
of a {\em pseudo-manifold}.
The Natarajan dimension $1$ condition is almost identical to 
the {\em flag-no-square} condition;
this condition was studied in many works
as a local combinatorial criteria for hyperbolicity.

As the abstract of~\citep*{januszkiewicz2003hyperbolic}
indicates,
simplicial complexes in the spirit we need were 
conjectured not to exist (by Moussong), or at least to require difficult number theory
(by Gromov).
\an{However,}
\cite*{januszkiewicz2003hyperbolic} built 
a simplicial complex that exactly meets our needs
(see Section~\ref{sec:polish}).

The difficulty of the construction 
explains the fact that Natarajan's question was open for so many years.
For example, for $d=2$,
the smallest concept class with Natarajan dimension $1$
we know of has size $6$; see Figure~\ref{fig:2dim}.
For $d=3$, the size grows to $54$; see Figure~\ref{fig:3dimpc}.
For $d=4$, the size jumps to $118,098$.
{This large complex is not the complex suggested
in~\citep*{januszkiewicz2003hyperbolic}.
The high-level structure of the construction is similar,
but the complex we found is smaller.
We found the construction and verified it with a computer
(using~\citep{gap4}).}
See Section~9 of~\citep*{januszkiewicz2003hyperbolic}
for more details on the ``complexity'' of their construction.

\subsection{Pseudo-cubes and Simplicial Complexes}\label{sec:PCvsSC}

We begin with a brief introduction to simplicial complexes.
Simplicial complexes are combinatorial abstractions
of triangulations of topological spaces.
A family $C$ of finite subsets of a set $V$ is called a simplicial complex if it is downward closed.
That is, for every $f\in C$, if $g\subset f$ then $g\in C$.
A member of $C$ is called a simplex or a face. 
The {\it dimension} of a face $f\in C$ is defined to be $\dim(f)=\lvert f\rvert -1$ and the dimension of the complex $C$ is $\dim(C)=\max_{f\in C}\dim(f)$.
A simplicial complex is called {\it pure} if all of its maximal faces have the same dimension.
The {\it $1$-skeleton}
of a simplicial complex $C$ is a graph
whose vertices are the elements of $V$ and 
whose edges are all the $1$-dimensional faces of $C$.
Every face in $C$ thus corresponds to a clique in its $1$-skeleton. 

We also need our complexes to be properly colored.
A {\it proper coloring} of a complex $C$ 
is a proper coloring of the 1-skeleton of $C$ with {$\dim(C)+1$ colors. 
That is, it is an assignment $r:V\to [\dim(C)+1]$} such that $r(u)\neq r(v)$ for every distinct $u,v$ so that $\{u,v\}\in C$. 

Our first goal in this subsection is to express the notion ``pseudo-cube'' in the language of simplicial complexes.
This is captured by the following definitions.
We say that a complex $C$ satisfies {\em replacement}
if for every simplex $f\in C$ and for every vertex $v\in f$ there exists a vertex $u\neq v$ such that $(f\setminus\{v\})\cup\{u\}\in C$.

\begin{definition}[Good Complex]
A simplicial complex $C$ is {\it good} if it is finite, pure, 
has a proper coloring, and satisfies replacement.
\end{definition}

The following proposition summarizes the equivalence between
pseudo-cubes and good simplicial complexes.
Figure~\ref{fig:3dimpc} may help in digesting this equivalence.

\begin{prop}[Concept Classes $\equiv$ Good Complexes]
For every $d$-dimensional good complex $C$ and a proper coloring $r$ of $C$, there is a $(d+1)$-dimensional pseudo-cube $B=B(C,r)$.
Conversely, for every $d$-dimensional pseudo-cube $B$,
there is a $(d-1)$-dimensional good simplicial complex
$C=C(B)$.
\end{prop}

\begin{remark}
\an{The pseudo-cube $B(C, r)$
and the complex $C(B)$ are explicitly constructed in the proof.}
\end{remark}

\begin{proof}
{\em Good complex $\implies$ pseudo-cube.}
Let $C$ be a good $d$-dimensional complex over $V$ and let $r:V\to [d+1]$ be a proper coloring of $C$.
We define a $(d+1)$-dimensional {pseudo-cube~$B\subseteq V^{d+1}$ as} follows. 
Each face $f\in C$ of maximum size $\lvert f\rvert = d+1$ defines 
the word $(v_1,\ldots,v_{d+1})\in B$
such that for each $i \in [d+1]$,
the vertex $v_i\in f$ is the unique vertex in $f$ with color $r(v_i)=i$.
It remains to verify that $B$ is a pseudo-cube. 
The class $B$ is finite because $C$ is finite.
For every $(v_1,\ldots, v_{d+1})\in B$ and~$i \in [d+1]$,
the following holds.
The set $\{v_1,\ldots v_{d+1}\}$ is a face in $C$.
By the replacement property, there exists $u\neq v_i$ such that $\{v_1,\ldots v_{i-1},u,v_{i+1},\ldots v_{d+1}\}$
is a face in $C$. 
Because $r$ is a proper coloring, it must be that $r(u)=i$.
So, $(v_1,\ldots v_{i-1},u,v_{i+1},\ldots v_{d+1})$ is in $B$ as well.

{\em Pseudo-cube $\implies$ good complex.}
Given a $d$-dimensional pseudo-cube $B\subseteq \Y^d$,
define a simplicial complex $C$ as follows.
The vertex-set of $C$ consists of all $(y,i) \in \Y \times [d]$
so that $y$ appears as the $i$'th letter of some word in $B$.
Each $(y_1,\ldots ,y_d)\in B$ defines 
the maximal face $\{(y_i,i) : i \in [d]\}$ of $C$.
It remains to verify that $C$ is good.
The complex $C$ is finite because $B$ is finite. 
The complex $C$ is pure and all its maximal faces have dimension $d-1$.
Consider the coloring $r:V\to [d]$ defined by $r((y,i))=i$.
It is a proper coloring
because all faces contain at most one vertex of each color.
Because $B$ is a pseudo-cube, for each $i \in [d]$ there exists a word $(y_1,\ldots ,y_{i-1},y'_i,y_{i+1},\ldots, y_d)\in B$
with $y'_i\neq y_i$.
In other words, $C$ satisfies the replacement property.
\end{proof}

%

The remaining of this section is about translating the Natarjan dimension
condition to the language of simplicial complexes.
A {\em square} $v_0\to v_1 \to v_2 \to v_3 \to v_0$ in a simplicial complex $C$
is a sequence of four distinct vertices that form a cycle of length four in the $1$-skeleton of $C$.

\begin{prop}[Natarajan Dimension 
for Colored Complex]\label{prop:goodnat1}
Let $C$ be a $d$-dimensional good complex and let $r$ be a proper coloring of $C$.
Let $B=B(C,r)$ be the pseudo-cube that is defined by $C,r$.
The two following properties are equivalent:
\begin{enumerate}
    \item There exists a square $v_0\to v_1 \to v_2 \to v_3 \to v_0$ in $C$ such that $r(v_0)=r(v_2)$ and $r(v_1)=r(v_3)$.
    \item The Natarajan dimension of $B$ is at least $2$.
\end{enumerate}
\end{prop}

\begin{proof}
First, assume that 
there exists a square $v_0\to v_1 \to v_2 \to v_3 \to v_0$ in $C$ such that $i:=r(v_0)=r(v_2)$ and $j:=r(v_1)=r(v_3)$.
Because $r$ is proper, we know $i \neq j$.
Without loss of generality, assume $i < j$. 
It suffices to show that $B\vert_{\{i,j\}}$ contains all 4 patterns
\[(v_0,v_1), (v_2,v_1), (v_0,v_3), (v_2,v_3).\]
This follows because each of the patterns $(v_a, v_b)$ for $a\in\{0,2\}$
and $b\in\{1,3\}$
corresponds to an edge $\{v_a,v_b\}$ in $C$.
This edge is contained in a maximal $d$-dimensional face $f_{a,b}\in C$.
By the definition of $B$, the face $f_{a,b}$ corresponds to a word in $B$ which realizes the pattern $(v_a,v_b)$ on $\{i,j\}$.

In the other direction,
assume that the Natarajan dimension of $B$ is at least $2$.
Let $i<j$ be a pair of coordinates such that 
there exist labels $v_0,v_1,v_2,v_3$
such that the following $4$ patterns belong to $B\vert_{\{i,j\}}$:
\[(v_0,v_1), (v_2,v_1), (v_0,v_3), (v_2,v_3).\]
The definition of $B$ thus implies that $r(v_0)=r(v_2)=i$, that $r(v_1)=r(v_3)=j$,
and that $\{v_0,v_1\},\{v_1,v_2\}, \{v_2,v_3\}, \{v_3,v_0\}\in C$.
So,  $v_0\to v_1\to v_2\to v_3\to v_0$ is the desired square.
\end{proof}

Typically, simplicial complexes are not colorful.
So, it is helpful to have a version of Proposition~\ref{prop:goodnat1}
that does not require a proper coloring.
An {\em empty square} in a simplicial complex $C$ is a square
$v_0\to v_1 \to v_2 \to v_3 \to v_0$
so that both
$\{v_0,v_2\}$ and $\{v_1,v_3\}$ are not edges in the $1$-skeleton of $C$.
In other words, an empty square is a square
so that the induced graph on its vertices
is the same square
(somewhat confusingly this is also known as a {\em full} square in some contexts). 

\begin{corollary}[Natarajan Dimension for Complex]
\label{cor:noSquareisN}
If there are no empty squares in a good simplicial complex $C$ of dimension $d$
then for every proper coloring $r$ of $C$, 
the Natarajan dimension of $B(C, r)$ is at most $1$.
\end{corollary}

\begin{proof}
By Proposition~\ref{prop:goodnat1},
if the Natarajan dimension of $B(C, r)$ is at least $2$, 
then there is a square $v_0\to v_1 \to v_2 \to v_3 \to v_0$ in $C$ such that $r(v_0)=r(v_2)$ and $r(v_1)=r(v_3)$.
Because $r$ is a proper coloring,
the square must be empty.
\end{proof}

\subsection{The Simplicial Complex} 
\label{sec:polish}

The goal of this section is to state the construction by~\cite*{januszkiewicz2003hyperbolic} of the simplicial complexes
we need.

How can we build a complex $C$,
that is pure, has a proper coloring and satisfies replacement?
This is quite easy, and we shall return to it below.
The reason is that we did not insist that $C$ is finite.
The challenge is to have all these properties in a finite object.

A baby version of this difficulty appears already in graph theory.
It is fairly easy to build an infinite regular tree,
but constructing finite regular graphs is more challenging.
Group theory provides a fundamental and powerful mechanism
to ``fold'' the infinite tree to a finite regular graph.
If the infinite tree is thought of as a Cayley graph of some group $F$,
and $N$ is a normal subgroup of $F$ of finite index,
then the ``modulo $N$'' operation allows to fold the tree to a finite graph.
Many useful constructions of finite graphs are obtained via this mechanism.

Coming back to an infinite complex that is pure,
properly colored and satisfies replacement,
we can simply start with a face of dimension $d$,
connect it to $d+1$ new faces by adding new vertices, and keep going indefinitely.
This construction corresponds to an infinite regular tree
(see also Example~\ref{ex:tree}).
It is easy to build, but utterly useless for us.
The real difficulty is to ``fold'' it to be finite.
What does ``fold'' even mean?
The solution is again algebraic,
but it uses the more abstract language of coset complexes.

Let $F$ be a group (finite or infinite).
A coset of a subgroup $H \leq F$ is 
a set of the form $g H = \{g h : h\in H\}$.
The {\it coset complex} defined by subgroups $H_1,\ldots, H_d\leq F$
is the simplicial complex $C=C_F(H_1,\ldots, H_d)$ that is defined as follows.
The vertices of $C$ are the cosets of the groups
$H_1,\ldots , H_d$, and a set of cosets $\sigma$
is a simplex in $C$ if and only if the intersection of all cosets in $\sigma$ is non-empty: $\sigma\in C \iff \bigcap_{L\in\sigma}L \neq\emptyset$.
Stated differently, the complex is the {\em nerve} of the set of all cosets.

The following theorem states the existence of the coset complexes we need.

\begin{theorem}[\cite*{januszkiewicz2003hyperbolic}]\label{thm:polish}
For every integer $d >1$, there exists a finite group $F$, and $d$ subgroups $H_1,\ldots, H_d\leq F$ such that the following hold:
\begin{enumerate}
    \item  For every $i\in[d]$, we have $(\cap_{j\neq i} H_j)\setminus H_i\neq\emptyset$.
    \item The coset complex $C_F(H_1,\ldots, H_d)$ does not contain empty squares.
\end{enumerate}
\end{theorem}

Theorem~\ref{thm:polish} is a consequence of a deep construction by~\cite*{januszkiewicz2003hyperbolic} which combines tools and ideas
from algebra and topology that are beyond the scope of our work. 
In Appendix~\ref{app:polish} we formally derive Theorem~\ref{thm:polish} using results stated in \citep*{januszkiewicz2003hyperbolic}. 
It is rather a formality, because all ideas are already in that paper,
but the exact result we need, unfortunately, is not explicitly stated.
This derivation is {\em not} self-contained and uses concepts that are defined in~\citep*{januszkiewicz2003hyperbolic}. 

Nevertheless, let us provide a simplified and high-level description 
of their approach.
The proof of Theorem~\ref{thm:polish} is by induction on $d$.
The group $F$ is generated by $d$ involutions $Z = \{z_1,\ldots,z_{d}\}$.
The involution condition $z^2 = 1$ corresponds to 
having a {\em single} neighbor in each coordinate in the corresponding concept class.
This is a strong version of the pseudo-cube condition.
The subgroup $H_i$ is generated by 
the $d-1$ involutions $Z \setminus \{z_i\}$.

How can we apply induction?
Imagine that each $H_i$ plays the role of $F$.
So, $H_i$ is finite group and it has $d-1$ subgroups that yield a 
pseudo-cube of dimension $d-1$.
These $d-1$ subgroups of $H_i$ are generated by
$Z \setminus \{z_i,z_j\}$ for each $j \in [d] \setminus \{i\}$.
In other words, we have $d$ constructions for dimension $d-1$
that are somehow ``glued together using algebra''.

All these finite groups naturally live inside one big group $\mathcal{F}$.
This group $\mathcal{F}$ is the free product of the $H_i$'s modulo the ``obvious''
relations that are there because e.g.~$z_1$ is the same
inside $H_2$ and inside $H_3$ (i.e., $\mathcal{F}$ is the free product with amalgamation). 
The group $\mathcal{F}$, however, is infinite.
We obtain an infinite simplicial complex
$C_\mathcal{F} = C_\mathcal{F}(H_1,\ldots,H_d)$.
Again, an infinite object that we need to make finite. 
To do so, we need to carefully identify a normal subgroup
of $\mathcal{F}$ of finite index, so that after we divide by it
``everything still works''.
Where should we look for this magical subgroup?

One key idea is to replace the no empty square condition 
by a stronger algebraic condition that the groups $H_i$
and their subgroups satisfy.
This condition is called {\em extra retractibility}. 
It states the existence of certain homomorphisms
between various subgroups of the $H_i$'s.
Isolating the extra retractibility condition is a major and difficult step.
Even verifying that extra retractibility implies the
no empty square condition is not at all trivial.
But the real question is: what did we gain from this algebraic move?

The gain is that we can try to prove that 
the infinite complex $C_{\mathcal{F}}$ is again extra retractible.
This is not the end of the story, but it is a start.
Proving that $C_{\mathcal{F}}$ is extra retractible uses topology as well as
the ``universality'' of free products.
Topological properties of the complex $C_{\mathcal{F}}$
(e.g., it is connected and simply connected)
allow to represent it in a ``non-obvious'' way.
The inductive hypothesis shows that 
extra retractibility ``locally holds'' on $C_{\mathcal{F}}$.
Universality now implies that it also holds globally,
due to the topological properties.

The infinite complex is extra retractible.
So what?
The algebraic nature of extra retractibility serves as a guide
in the search for the magical normal subgroup.
Extra retractibility states the existence of certain homomorphisms (to finite groups).
We can identify a finite index subgroup $\mathcal{K}$ of $\mathcal{F}$
that is contained in all kernels of these homomorphisms.
Taking the normal core $\mathcal{N} = \bigcap_{g \in \mathcal{F}} g \mathcal{K} g^{-1}$
yields a normal subgroup of finite index (!)
that is contained in ``all kernels''. 
This latter property of $\mathcal{N}$ tells us that even after we divide by it,
extra retractibility still holds.
We can finally choose $F = \mathcal{F}/\mathcal{N}$ and complete the inductive step.

Let us return to the main goal of this section,
deducing the needed concept class
from the construction of~\cite*{januszkiewicz2003hyperbolic}.

\begin{prop}[There is a Good Complex
with No Empty Squares]
\label{prop:polishreduce}
Let $F$ and $H_1,\ldots, H_d$ be as in Theorem~\ref{thm:polish}. 
The coset complex
$C=C_F(H_1,\ldots, H_d)$ has dimension $d-1$, is good and
has no empty squares.
\end{prop}

\begin{proof}
%
Because $F$ is finite, $C$ is finite as well.

Let us prove that $C$ is pure of dimension $d-1$.
If $\sigma$ is a face,
then there is $g \in \bigcap_{L\in\sigma}L$.
Because every two distinct cosets of the same subgroup are disjoint,
the face $\sigma$ is of the form $\{g H_i : i \in I\}$
for some $I \subseteq [d]$.
The face $\sigma$ is contained in the maximal 
face $\{g H_i : i \in [d]\}$ which has dimension $d-1$.

There is a straightforward proper coloring of $C$ with $d$ colors.
Color each vertex of the form $g H_i$ by the color $i$.
{This is a coloring because a subset of a group
can be a coset of at most one subgroup.}
This is a proper coloring because two distinct cosets of the same subgroup are disjoint.
%

Finally, we prove that $C$ satisfies replacement.
Let $\sigma \in C$. 
As above, we can write $\sigma = \{g H_i : i \in I\}$ for some $I \subseteq [d]$.
Consider a vertex $gH_k$ inside $\sigma$.
By construction of $C$,
let $z \in (\cap_{j\neq k} H_j)\setminus H_k$.
For every $j \neq k$,
we have $g z H_j = gH_j$.
Because $z \not \in H_k$,
we can conclude $g z H_k \neq gH_k$.
It follows that $(\sigma \setminus \{g H_k\}) \cup \{g z H_k\}
= \{g z H_i : i \in I\}$
is also a face.

%
%
%
\end{proof}

\subsection{Wrapping up}\label{sec:wrapup}
\begin{proof}[Proof of Theorem~\ref{thm:lb}]
Theorem~\ref{thm:polish}, Proposition~\ref{prop:polishreduce}
and Corollary~\ref{cor:noSquareisN}
imply that for every $d$, there exists a $d$-dimensional pseudo-cube $B_d\subseteq Y_d^{X_d}$ where $|X_d|=d$ with Natarajan dimension $1$.
We may assume that the label-sets $Y_d$'s are pairwise disjoint,
and that the domains $X_d$'s are pairwise disjoint.

Construct the ``disjoint union'' of all these classes.
Let $\X=\bigcup_{d} X_d$.
Each $h \in B_d$ is a partial map on $\X$
because it is defined only on $X_d$.
Let $\star$ be a new label such that $\star\notin \bigcup_d Y_d$.
Extend each such $h$ by setting it to be $\star$ outside $X_d$. 
Denote by $H_d$ the collection of extensions of maps in $B_d$.
Finally, let 
$$\H=\bigcup_d H_d.$$
By construction, 
the DS dimension of $\H$ is at least $d$ for every integer $d$
because it contains a copy of $B_d$.
%

It remains to prove that the Natarajan dimension of $\H$ is $1$.
It is at least $1$ because $\lvert \H\rvert \geq 2$. 
The last thing to verify is that the Natarajan dimension is less than $2$.
Let $x_1,x_2\in \X$ be a pair of distinct points and assume towards contradiction that $\{x_1,x_2\}$ is N-shattered by $\H$.
If $x_1\in X_{d_1}$ and $x_2\in X_{d_2}$ for $d_1\neq d_2$,
then every function $h\in \H$ satisfies that $\star\in\{h(x_1),h(x_2)\}$ and therefore $\{x_1,x_2\}$ is not N-Shattered by $\H$.
The last remaining case is that $x_1,x_2\in X_{d}$ for the same $d$. 
In this case, every function $h\in \H\setminus H_d$ satisfies that $h(x_1)=h(x_2)=\star$ and there is no function $h\in \H$ such that $\{h(x_1),h(x_2)\}= \{\star, y\} $ for $y\neq \star$.
It follows that $\{x_1,x_2\}$ must be N-shattered by $H_d$ which is a contradiction because the Natarajan dimension of $B_d$ is $1$.
\end{proof}

\bibliographystyle{plainnat}
\bibliography{bib}

\begin{thebibliography}{28}
\providecommand{\natexlab}[1]{#1}
\providecommand{\url}[1]{\texttt{#1}}
\expandafter\ifx\csname urlstyle\endcsname\relax
  \providecommand{\doi}[1]{doi: #1}\else
  \providecommand{\doi}{doi: \begingroup \urlstyle{rm}\Url}\fi

\bibitem[Alon et~al.(2021)Alon, Hanneke, Holzman, and Moran]{alon2021theory}
Noga Alon, Steve Hanneke, Ron Holzman, and Shay Moran.
\newblock A theory of {PAC} learnability of partial concept classes.
\newblock \emph{arXiv:2107.08444}, 2021.

\bibitem[Ben-David et~al.(1995)Ben-David, Cesabianchi, Haussler, and
  Long]{bendavid1995characterizations}
Shai Ben-David, Nicolo Cesabianchi, David Haussler, and Philip~M Long.
\newblock Characterizations of learnability for classes of \{0,...,n\}-valued
  functions.
\newblock \emph{Journal of Computer and System Sciences}, 50\penalty0
  (1):\penalty0 74--86, 1995.

\bibitem[Blumer et~al.(1989)Blumer, Ehrenfeucht, Haussler, and
  Warmuth]{blumer:89}
Anselm Blumer, Andrzej Ehrenfeucht, David Haussler, and Manfred~K. Warmuth.
\newblock Learnability and the {Vapnik-Chervonenkis} dimension.
\newblock \emph{Journal of the ACM}, 36\penalty0 (4):\penalty0 929--965, 1989.

\bibitem[Brukhim et~al.(2021)Brukhim, Hazan, Moran, and
  Schapire]{brukhim2021multiclass}
Nataly Brukhim, Elad Hazan, Shay Moran, and Robert~E. Schapire.
\newblock Multiclass boosting and the cost of weak learning.
\newblock In \emph{NIPS}, 2021.

\bibitem[Daniely and Shalev-Shwartz(2014)]{daniely2014optimal}
Amit Daniely and Shai Shalev-Shwartz.
\newblock Optimal learners for multiclass problems.
\newblock In \emph{COLT}, pages 287--316, 2014.

\bibitem[Daniely et~al.(2012)Daniely, Sabato, and
  Shalev{-}Shwartz]{DanielySS12}
Amit Daniely, Sivan Sabato, and Shai Shalev{-}Shwartz.
\newblock Multiclass learning approaches: {A} theoretical comparison with
  implications.
\newblock In \emph{NIPS}, pages 494--502, 2012.

\bibitem[Daniely et~al.(2015{\natexlab{a}})Daniely, Sabato, Ben-David, and
  Shalev-Shwartz]{daniely2015multiclass}
Amit Daniely, Sivan Sabato, Shai Ben-David, and Shai Shalev-Shwartz.
\newblock Multiclass learnability and the {ERM} principle.
\newblock \emph{The Journal of Machine Learning Research}, 16:\penalty0
  2377--2404, 2015{\natexlab{a}}.

\bibitem[Daniely et~al.(2015{\natexlab{b}})Daniely, Schapira, and
  Shahaf]{daniely2015inapproximability}
Amit Daniely, Michael Schapira, and Gal Shahaf.
\newblock Inapproximability of truthful mechanisms via generalizations of the
  vc dimension.
\newblock In \emph{STOC}, pages 401--408, 2015{\natexlab{b}}.

\bibitem[David et~al.(2016)David, Moran, and Yehudayoff]{david2016supervised}
Ofir David, Shay Moran, and Amir Yehudayoff.
\newblock Supervised learning through the lens of compression.
\newblock In \emph{NIPS}, pages 2784--2792, 2016.

\bibitem[GAP(2021)]{gap4}
GAP.
\newblock The gap group, gap -- groups, algorithms, and programming.
\newblock 2021.

\bibitem[Haussler(1992)]{Haussler92}
David Haussler.
\newblock Decision theoretic generalizations of the {PAC} model for neural net
  and other learning applications.
\newblock \emph{Inf. Comput.}, 100\penalty0 (1):\penalty0 78--150, 1992.

\bibitem[Haussler(1995)]{haussler1995sphere}
David Haussler.
\newblock Sphere packing numbers for subsets of the boolean n-cube with bounded
  {V}apnik-{C}hervonenkis dimension.
\newblock \emph{Journal of Combinatorial Theory, Series A}, 69\penalty0
  (2):\penalty0 217--232, 1995.

\bibitem[Haussler and Long(1995)]{HausslerL95}
David Haussler and Philip~M. Long.
\newblock A generalization of {S}auer's lemma.
\newblock \emph{J. Comb. Theory, Ser. {A}}, 71\penalty0 (2):\penalty0 219--240,
  1995.

\bibitem[Haussler et~al.(1994)Haussler, Littlestone, and
  Warmuth]{haussler1994predicting}
David Haussler, Nick Littlestone, and Manfred~K Warmuth.
\newblock Predicting $\{$0, 1$\}$-functions on randomly drawn points.
\newblock \emph{Information and Computation}, 115\penalty0 (2):\penalty0
  248--292, 1994.

\bibitem[Januszkiewicz and
  {\'S}wi{\k{a}}tkowski(2003)]{januszkiewicz2003hyperbolic}
Tadeusz Januszkiewicz and Jacek {\'S}wi{\k{a}}tkowski.
\newblock Hyperbolic coxeter groups of large dimension.
\newblock \emph{Commentarii Mathematici Helvetici}, 78\penalty0 (3):\penalty0
  555--583, 2003.

\bibitem[Littlestone and Warmuth(1986)]{littlestone1986relating}
Nick Littlestone and Manfred Warmuth.
\newblock Relating data compression and learnability.
\newblock \emph{Unpublished manuscript}, 1986.

\bibitem[Natarajan(1988)]{Natarajan88up}
Balas~K. Natarajan.
\newblock Some results on learning.
\newblock \emph{Unpublished manuscript}, 1988.

\bibitem[Natarajan(1989)]{natarajan1989learning}
Balas~K. Natarajan.
\newblock On learning sets and functions.
\newblock \emph{Machine Learning}, 4\penalty0 (1):\penalty0 67--97, 1989.

\bibitem[Natarajan and Tadepalli(1988)]{NatarajanT88}
Balas~K. Natarajan and Prasad Tadepalli.
\newblock Two new frameworks for learning.
\newblock In \emph{ICML}, pages 402--415, 1988.

\bibitem[Pollard(1990)]{Pollard90}
David Pollard.
\newblock Empirical processes: Theory and applications.
\newblock \emph{NSF-CBMS Regional Conference Series in Probability and
  Statistics}, 2:\penalty0 1--86, 1990.

\bibitem[Rubinstein et~al.(2006)Rubinstein, Bartlett, and
  Rubinstein]{rubinstein2006shifting}
Benjamin Rubinstein, Peter Bartlett, and J~Hyam Rubinstein.
\newblock Shifting, one-inclusion mistake bounds and tight multiclass expected
  risk bounds.
\newblock In \emph{NIPS}, pages 1193--1200, 2006.

\bibitem[Schapire and Freund(2012)]{schapire2012boosting}
Robert~E Schapire and Yoav Freund.
\newblock \emph{Boosting: Foundations and algorithms}.
\newblock Cambridge University Press, 2012.

\bibitem[Shalev-Shwartz and Ben-David(2014)]{shalev2014understanding}
Shai Shalev-Shwartz and Shai Ben-David.
\newblock \emph{Understanding machine learning: From theory to algorithms}.
\newblock Cambridge University Press, 2014.

\bibitem[Valiant(1984)]{Valiant84}
Leslie~G. Valiant.
\newblock A theory of the learnable.
\newblock In \emph{STOC}, pages 436--445, 1984.

\bibitem[Vapnik(1989)]{Vapnik89}
Vladimir Vapnik.
\newblock Inductive principles of the search for empirical dependences (methods
  based on weak convergence of probability measures).
\newblock In \emph{COLT}, pages 3--21, 1989.

\bibitem[Vapnik and Chervonenkis(1968)]{Vapnik68}
Vladimir Vapnik and Alexey Chervonenkis.
\newblock On the uniform convergence of relative frequencies of events to their
  probabilities.
\newblock In \emph{Proc. USSR Acad. Sci.}, 1968.

\bibitem[Vapnik and Chervonenkis(1974)]{vapnik:74}
Vladimir Vapnik and Alexey Chervonenkis.
\newblock \emph{Theory of Pattern Recognition}.
\newblock Nauka, Moscow, 1974.

\bibitem[von Neumann(1928)]{neumann1928theorie}
John von Neumann.
\newblock Zur theorie der gesellschaftsspiele.
\newblock \emph{Mathematische Annalen}, 100\penalty0 (1):\penalty0 295--320,
  1928.

\end{thebibliography}

\appendix

\section{The Simplicial Complex}\label{app:polish}

This parts explains how to deduce
Theorem~\ref{thm:polish} from the results that are stated in~\citep*{januszkiewicz2003hyperbolic}.
This appendix uses definitions and theorems from that paper.
The key idea is to construct a {\it development} $L_d$
of an {\it extra retractible} complex of finite groups on the poset
$\{0,1\}^{[d]}$;
see Section 6 in~\citep{januszkiewicz2003hyperbolic}.
The complex of groups consists of a monotone mapping from subsets of $[d]$ to finite groups. Each $A\subseteq [d]$ is assigned a finite group $F_A$ such that $F_A \leq F_B$ whenever $A\subseteq B$.  
Extra retractibility further implies that
\begin{align}
    \forall A,B\subseteq [d] \ \ \ F_A\cap F_B & = F_{A\cap B}\label{eq:cap}\\
    F_{\emptyset} & =\{1\}\label{eq:unit}
\end{align}
Equation~\ref{eq:cap} follows from Propositions 3.2 and 4.1 in \citep{januszkiewicz2003hyperbolic}.
Equation~\ref{eq:unit} follows because an extra retractible complex of groups is {\it reduced}; see Definitions~4.4 and~5.8 in \cite{januszkiewicz2003hyperbolic}.

In the construction, the groups are generated by involutions.
For each $i\in[d]$ the group~$F_{\{i\}}$ 
is $\{1,z_i\}$ where $z_i^2=1$.
To prove Theorem~\ref{thm:polish}, we set 
\begin{center}
$F=F_{[d]}$ and $H_i=F_{[d]\setminus\{i\}}$.
\end{center}
With these choices, the development $L_d$ is isomorphic to the coset complex $C=C(H_1,\ldots, H_d)$.

{To justify Item 1, we shall prove} that $z_i\in (\cap_{j\neq i}H_j)\setminus H_i$.
Because $z_i\in F_{\{i\}}$
and by Equation~\ref{eq:cap},
\[F_{\{i\}}\cap H_i = F_{\{i\}}\cap F_{[d]\setminus\{i\}} = F_{\emptyset}=\{1\} \implies z_i\notin H_i.\]
On the other hand, for $j \neq i$,
\[z_i\in F_{\{i\}}\subseteq F_{[d]\setminus\{j\}}=H_j.\]

Finally, Proposition 5.12 in \citep{januszkiewicz2003hyperbolic} asserts that $C$ contains no empty squares. 

\section{Orientations for infinite graphs}
\label{app:comapct}

\an{Here we complete the proof of Lemma~\ref{lem:orientation}
for infinite graphs.}
Let $\G=(V,E)$ be the one-inclusion graph of $\H$.
Let $\Z$ be the set of pairs $z=(v, e) \in V \times E$ so that $v \in e$.
Let $\K = \{0,1\}^\Z$.
Tychonoff's theorem says that $\K$ is compact
with respect to the product topology.

An orientation corresponds to choosing for each $e \in E$
a single $v \in V$.
In other words, each orientation can be thought of as 
an element of $\K$,
where $\kappa_{(v, e)} = 1$
means that $e$ is oriented towards $v$, and
$\kappa_{(v, e)} = 0$
means that $e$ is not oriented towards $v$.

For every $v \in V$, let 
$A_v \subseteq \K$ be the set of all $\kappa \in \K$
so that there are at most $d$ edges $e \ni v$ so that
$\kappa_{(v, e)} =0$.
For $j \in [d+1]$,
let $B_{v,j} \subseteq \K$ be the set of all $\kappa \in \K$
so that for the edge $e = e_j$ that is in the $j$'th direction of $v$,
there is at most one $u$ so that $\kappa_{(u, e)} =1$.

The complement of the set $A_v$ is open because
it is
$\{\kappa \in \K : \forall i \in [d+1] \ \kappa_{(v,e_i)} =0\}$,
where $e_i$ is the edge in the $i$'th direction of $v$.
The complement of the set $B_{v,j}$ is open because
it is the union over all sets $\{w_1,w_2\}$
of two vertices that are contained in $e = e_j$
of $\{\kappa \in \K : \kappa_{(w_1,e)} =
\kappa_{(w_2,e)} =1 \}$.
The set $\Sigma_v = A_v \cap \bigcap_{j \in [d+1]}B_{v,j}$ is hence closed.

We now claim that for every finite
$U \subset V$ the set $\bigcap_{v \in U} \Sigma_v$ is non-empty.
The finite hyper-graph $\G_U$ that $\G(\H)$ induces on $U$
has an orientation $\sigma$ with maximum out-degree at most $d$.
The orientation $\sigma$ defines an element $\kappa$ in $\bigcap_{v \in U} \Sigma_v$ as follows.
There are two types of edges:
edges in $\G(\H)$ that correspond to 
edges in $\G_U$, and edges that ``disappear'' with the projection to~$U$.
The former type of edges are oriented in $\kappa$ exactly as in $\sigma$.
The latter type of edges are edges $e$
{that contain at most one vertex in $U$ and possibly
other vertices outside $U$.
If $e$ contains a single vertex $u$ from $U$}
then orient $e$ towards~$u$,
and otherwise orient $e$ to an arbitrary element (not in $U$).

Compactness now implies that the intersection of 
all the $\Sigma_v$'s is non-empty.
In particular, there is $\kappa^*$ in $\bigcap_{v \in V} \Sigma_v$.
This $\kappa^*$ can be thought of as a partial orientation,
because every edge $e$ contains some vertex
(and so there is at most one $u \in e$ so that $\kappa^*_{(u, e)}=1$).
For each $v \in V$, there are at most $d$ edges $e \ni v$ so that
$\kappa^*_{(v, e)} =0$.
Complete $\kappa^*$ to a full orientation $\sigma^*$
by arbitrarily orienting all edges that are not oriented in $\kappa^*$.
The out-degree of $\sigma^*$ is still at most $d$,
because the final move from $\kappa^*$ to $\sigma^*$ 
{does not increase out-degrees.}

\end{document}